\documentclass[11pt]{article}
\usepackage{fullpage}

\usepackage{amsfonts}

\usepackage{natbib}
\usepackage{epsfig}
\usepackage{amsthm}
\usepackage{amsmath}
\usepackage{amssymb}
\usepackage{algorithm}
\usepackage{algorithmic}
\usepackage{subfig}
\usepackage{graphicx}
\usepackage{color}
\usepackage[section] {placeins}
\usepackage{epstopdf}

\usepackage{hyperref}

\usepackage[section] {placeins}

\newcounter{thm_counter}
\newcounter{lem_counter}

\newcounter{ass_counter}

\newtheorem{theorem}[thm_counter]{Theorem}
\newtheorem{lemma}[lem_counter]{Lemma}
\newtheorem{assumption}[ass_counter]{Assumption}

\newcommand{\rca}{\color{black}}
\newcommand{\rcb}{\color{black}}
\newcommand{\rr}{\color{black}}
\newcommand{\rb}{\color{black}}
\newcommand{\rcc}{\color{black}}

\title{Dictionary LASSO: Guaranteed Sparse Recovery under Linear Transformation}
\author{Ji Liu$^{\dag}$,~Lei Yuan$^{\ddag}$, and Jieping Ye$^{\ddag}$\\
ji-liu@cs.wisc.edu\quad \{lei.yuan, jieping.ye\}@asu.edu\\
$^{\dag}$Department of Computer Sciences, University of Wisconsin-Madison\\
$^{\ddag}$Department of Computer Science and Engineering, Arizona State University}

\begin{document}
\maketitle
\thispagestyle{plain} \pagestyle{plain}

\begin{abstract}
We consider the following signal recovery problem: given a
measurement matrix $\Phi\in \mathbb{R}^{n\times p}$ and a noisy
observation vector $c\in \mathbb{R}^{n}$ constructed from $c =
\Phi\theta^* + \epsilon$ where $\epsilon\in \mathbb{R}^{n}$ is the
noise vector whose entries follow i.i.d. centered sub-Gaussian
distribution, how to recover the signal $\theta^*$ if $D\theta^*$ is
sparse {\rca under a linear transformation} $D\in\mathbb{R}^{m\times
p}$ ({\rcc we call $D$ by dictionary matrix})? One natural method using convex optimization is to solve the
following problem: $$\min_{\theta}~{1\over 2}\|\Phi\theta - c\|^2 +
\lambda\|D\theta\|_1$${\rcc which is referred to as the ``dictionary LASSO''.} This paper provides an upper bound of the estimate error and shows the consistency property of {\rcc the dictionary LASSO} by
assuming that the design matrix $\Phi$ is a Gaussian random matrix.
Specifically, we show 1) in the noiseless case, if the condition
number of $D$ is bounded and the measurement number $n\geq
\Omega(s\log(p))$ where $s$ is the sparsity number, then the true
solution can be recovered with high probability; and 2) in the noisy
case, if the condition number of $D$ is bounded and the measurement
increases faster than $s\log(p)$, that is, $s\log(p)=o(n)$, the
estimate error converges to zero with probability 1 when $p$ and $s$
go to infinity. Our results are consistent with those for the
special case $D=\bold{I}_{p\times p}$ (equivalently LASSO) and
improve the existing analysis. The condition number of $D$ plays a
critical role in our analysis. We consider the condition numbers in
two cases including the fused LASSO and the random graph: the
condition number in the fused LASSO case is bounded by a constant,
while the condition number in the random graph case is bounded with
high probability if $m\over p$ (i.e., $\#\text{edge}\over
\#\text{vertex}$) is larger than a certain constant. Numerical
simulations are consistent with our theoretical results.
\end{abstract}


\section{Introduction}
The sparse signal recovery problem has been well studied recently
from the theory aspect to the application aspect in many areas
including compressive sensing \citep{candes07b, candes07a},
statistics \citep{Meinshausen06, Ravikumar08, bunea07, Lounici08, Koltchinskii08},
machine learning \citep{zhao06, zhang09b, wainwright09, LiuJMLR12},
and signal processing \citep{Romberg08,Donoho06, zhang09a}. {\rcb The
key idea is to use the $\ell_1$ norm to relax the $\ell_0$ norm (the
number of nonzero entries). This paper considers a specific type of
sparse signal recovery problems, that is, the signal is assumed to
be sparse under a linear transformation $D$.} It includes the
well-known fused LASSO \citep{TibshiraniSRZK05} as a special case.
The theoretical property of such problem has not been well
understood yet, although it has achieved success in many
applications \citep{Chan98, TibshiraniSRZK05, CandesRT06,
SharpnackRS12}. Formally, we define the problem as follows: given a
measurement matrix $\Phi\in \mathbb{R}^{n\times p}$ ($p\gg n$) and a
noisy observation vector $c\in \mathbb{R}^{n}$ constructed from $c =
\Phi \theta^* + \epsilon$ where $\epsilon\in \mathbb{R}^{n}$ is the
noise vector whose entries follow i.i.d. centered sub-Gaussian
distribution\footnote{\rca Note that this ``identical distribution''
assumption can be removed; see \citet{zhang09a}. For simplification
of analysis, we enforce this condition throughout this paper.}, how
to recover the signal $\theta^*$ if $D\theta^*$ is sparse where
$D\in\mathbb{R}^{m\times p}$ is a constant matrix dependent on the
specific application\footnote{\rcb We study the most general case of
$D$, and thus our analysis is applicable for both $m\geq p$ or
$m\leq p$.}? A natural model for such type of sparsity recovery
problems is:
\begin{align}
\min_{\theta}:~{1\over 2}\|\Phi \theta - c\|^2 +
\lambda\|D\theta\|_0.
\end{align}
The least square term is from the sub-Gaussian noise assumption and
the second term is due to the sparsity requirement. Since this
combinatorial optimization problem is NP-hard, the conventional
$\ell_1$ relaxation technique can be applied to make it tractable,
resulting in the following convex model:
\begin{align}
\min_\theta:~{1\over 2}\|\Phi \theta - c\|^2 +
\lambda\|D\theta\|_1.\label{eqn_org_ls}
\end{align}
Such model includes many well-known sparse formulations as special
cases:
\begin{itemize}
\item The fused LASSO \citep{TibshiraniSRZK05, FriedmanHHT07} solves
\begin{equation}\label{eqn_fusedLASSO}
\min_{\theta}:~{1\over 2}\|\Phi\theta - c\|^2 + \lambda_1
\|\theta\|_1 + \lambda_2 \|Q\theta\|_1
\end{equation}
where $Q\in\mathbb{R}^{(p-1)\times p}$ is defined as
the total variance matrix $Q=[\bold{I}_{(p-1)\times (p-1)}; \bold{0}_{p-1}] - [\bold{0}_{p-1};
\bold{I}_{(p-1)\times (p-1)}]$, that is,
\[ Q=\left[\begin{array}{ccccc}
    1 & -1 & 0 & ... & 0 \\
    0 & 1 & -1 & ... & 0 \\
    ... & ... & ... & ... & ... \\
    0 & 0 &... & 1 & -1\end{array} \right].\]
One can write Eq.~\eqref{eqn_fusedLASSO} in the form of
Eq.~\eqref{eqn_org_ls} by letting $\lambda=1$ and $D$ be the
conjunction of the identity matrix and the total variance matrix,
that
is, $$D = \left[\begin{array}{c}\lambda_1\bold{I}_{p\times p} \\
    \lambda_2Q\end{array}\right].$$
\item The general $K$ dimensional changing point detection problem
  \citep{CandesRT06, Needell2012, Needell2012m} can be expressed by
\begin{equation}
\begin{aligned}
  &\min_{\theta}:~{1\over 2}\sum_{(i_1,i_2,\cdots, i_K)\in S}(\theta_{i_1,i_2,\cdots, i_K}-c_{i_1,i_2,\cdots, i_K})^2 +\\
  &\lambda \sum_{i_1=1}^{I_1-1}\sum_{i_2=1}^{I_2-1}\cdots\sum_{i_K=1}^{I_K-1}(|\theta_{i_1,i_2,\cdots, i_K}-\theta_{i_1+1,i_2,\cdots,
  i_K}|+\\&
  \cdots + |\theta_{i_1,i_2,\cdots, i_K}-\theta_{i_1,i_2,\cdots,
    i_K+1}|)
\end{aligned}
\label{eqn_totalvariance}
\end{equation}
where $\theta\in\mathbb{R}^{I_1\times I_2\times\cdots I_K}$ is a $K$
dimensional tensor with a stepwise structure and $S$ is the set of
indices. The second term is used to measure the total variance. The
changing point is defined as the point where the signal changes. One
can properly define $D$ to rewrite Eq.~\eqref{eqn_totalvariance} in
the form of Eq.~\eqref{eqn_org_ls}. In addition, if the structure of
the signal is piecewise constant, then one can replace the second
term by
\begin{align*}
  &\lambda
  \sum_{i_1=2}^{I_1-1}\sum_{i_2=2}^{I_2-1}\cdots\sum_{i_K=2}^{I_K-1}
  (|2\theta_{i_1,i_2,\cdots, i_K}-\theta_{i_1+1,i_2,\cdots, i_K}-\theta_{i_1-1,i_2,\cdots, i_K}| \\
  &+ \cdots +
   |2\theta_{i_1,i_2,\cdots, i_K}-\theta_{i_1,i_2,\cdots,
    i_K+1}-\theta_{i_1,i_2,\cdots, i_K-1}|).
\end{align*}
It can be written in the form of Eq.~\eqref{eqn_org_ls} as well.
\item {\rcb The second term of
\eqref{eqn_totalvariance}, that is, the total variance, is defined
as the sum of differences between two neighboring entries (or
nodes). A graph can generalize this definition by using edges to
define neighboring entries rather than entry indexes.} Let $G(V,E)$
be a graph. One has
\begin{align}
\min_{\theta\in \mathbb{R}^{|V|}}:~{1\over 2}\|\Phi \theta - c\|^2 +
\lambda\sum_{(i,j)\in E}|\theta_i-\theta_j|,\label{eqn_randomgraph}
\end{align}
{\rcb where $\sum_{(i,j)\in E}|\theta_i-\theta_j|$ defines the total
variance over the graph $G$.} The $k^{th}$ edge between nodes $i$
and $j$ corresponds to the $k^{th}$ row of the matrix $D\in
\mathbb{R}^{|E|\times |V|}$ with zero at all entries except
$D_{ki}=1$ and $D_{kj}=-1$. Taking $\Phi = \bold{I}_{p\times p}$,
one obtains the edge LASSO \citep{SharpnackRS12}.
\end{itemize}

This paper studies the theoretical properties of {\rcc the dictionary LASSO in~\eqref{eqn_org_ls}} by providing an upper bound of the
estimate error, that is, $\|\hat\theta-\theta^*\|$ where
$\hat{\theta}$ denotes the estimation. The consistency property of
this model is shown by assuming that the design matrix $\Phi$ is a
Gaussian random matrix. Specifically, we show 1) in the noiseless
case, if the condition number of $D$ is bounded and the measurement
number $n\geq \Omega(s\log(p))$ where $s$ is the sparsity number,
then the true solution can be recovered under some mild conditions
with high probability; and 2) in the noisy case, if the condition
number of $D$ is bounded and the measurement number increases faster
than $s\log (p)$, that is, $n=O(s\log(p))$, then the estimate error
converges to zero with probability 1 under some mild conditions when
$p$ goes to infinity. Our results are consistent with those for the
special case $D=\bold{I}_{p\times p}$ (equivalently LASSO) and {\rcb
improve the existing analysis in \citet{CandesENR10, VaiterPGC12}.
To the best of our knowledge, this is the first work that
establishes the consistency properties for the general
problem~\eqref{eqn_org_ls}.} The condition number of $D$ plays a
critical role in our analysis. We consider the condition numbers in
two cases including the fused LASSO and the random graph: the
condition number in the fused LASSO case is bounded by a constant,
while the condition number in the random graph case is bounded with
high probability if $m\over p$ (that is,
$\#\text{edge}\over\#\text{vertex}$) is larger than a certain
constant. Numerical simulations are consistent with our theoretical
results.

\subsection{Notations and Assumptions}
Define
\begin{align*}
\rho^+_{\Psi, Y}(l_1,l_2) =&
\max_{h\in\mathbb{R}^{l_1}\times\mathcal{H}(Y,l_2)}{\|\Psi
h\|^2\over \|h\|^2},\\
 \rho^-_{\Psi, Y}(l_1,l_2) =& \min_{h\in
\mathbb{R}^{l_1}\times \mathcal{H}(Y,l_2)}{\|\Psi h\|^2\over
\|h\|^2},
\end{align*}
where $l_1$ and $l_2$ are nonnegative integers, $Y$ is the dictionary matrix, and $\mathcal{H}(Y,l_2)$ is the union of all subspaces spanned by
$l_2$ columns of $Y$:
$${\rcb\mathcal{H}(Y,l_2)=\{Yv~|~\|v\|_0\leq l_2\}.}$${\rcb Note
that the length of $h$ is the sum of $l_1$ and the number of rows of $Y$ (which is in general not equal to $l_2$). {\rb The definition of
$\rho^+_{\Psi, Y}(l_1,l_2)$ and $\rho^-_{\Psi, Y}(l_1,l_2)$ is inspired by
the D-RIP constant in \cite{CandesENR10}. Recall that the D-RIP constant $\delta_d$ is defined by the smallest quantity such that
$$(1-\delta_d)\|h\|^2\leq\|\Psi h\|^2 \leq (1+\delta_d)\|h\|^2\quad \forall~h\in \mathcal{H}(Y,l_2).$$One can verify that $\delta_d=\max \{\rho^+_{\Psi, Y}(0, l_2)-1, 1-\rho^-_{\Psi, Y}(0,
l_2)\}$ if $\Psi$ satisfies the D-RIP condition in terms of the
sparsity $l_2$ and the dictionary $Y$.}} Denote $ \rho^+_{\Psi, Y}(0,
l_2)$ and $\rho^-_{\Psi, Y}(0,l_2)$ as $\rho^+_{\Psi, Y}(l_2)$ and
$\rho^-_{\Psi, Y}(l_2)$ respectively for short.

Denote the compact singular value decomposition (SVD) of $D$ as
$D=U\Sigma V_\beta^T$. Let $Z=U\Sigma$ and its pseudo-inverse be
$Z^+=\Sigma^{-1}U^T$. One can verify that $Z^+Z=I$.
$\sigma_{\min}(D)$ denotes the minimal nonzero singular value of $D$
and $\sigma_{\max}(D)$ denotes the maximal one, that is, the
spectral norm $\|D\|$. One has $\sigma_{\min}(D)=\sigma_{\min}(Z) =
\sigma_{\max}^{-1}(Z^+)$ and
$\sigma_{\max}(D)=\sigma_{\max}(Z)=\sigma_{\min}^{-1}(Z^+)$. Define
$$\kappa := {\sigma_{\max}(D)\over \sigma_{\min}(D)}={\sigma_{\max}(Z)\over \sigma_{\min}(Z)}.$$ Let $T_0$ be
the support set of $D\theta^*$, that is, a subset of
$\{1,2,\cdots,m\}$, with $s:=|T_0|$. Denote $T_0^c$ as its
complementary index set with respect to $\{1,2,\cdots,m\}$. Without
loss of generality, we assume that $D$ does not contain zero rows.
Assume that $c=\Phi \theta +\epsilon$ where $\epsilon \in
\mathbb{R}^{n}$ and all entries $\epsilon_i$'s are i.i.d. centered
sub-Gaussian random variables with sub-Gaussian norm $\Delta$
(Readers who are not familiar with the sub-Gaussian norm can treat
$\Delta$ as the standard derivation in Gaussian random variable). In
discussing the dimensions of the problem and how they are related to
each other in the limit (as $n$ and $p$ both approach $\infty$), we
make use of order notation. If $\alpha$ and $\beta$ are both
positive quantities that depend on the dimensions, we write $\alpha
= O(\beta)$ if $\alpha$ can be bounded by a fixed multiple of
$\beta$ for all sufficiently large dimensions. We write $\alpha =
o(\beta)$ if for {\em any} positive constant $\phi>0$, we have
$\alpha \le \phi \beta$ for all sufficiently large dimensions. We
write $\alpha = \Omega(\beta)$ if both $\alpha=O(\beta)$ and
$\beta=O(\alpha)$. Throughout this paper, a Gaussian random matrix
means that all entries follow i.i.d. standard Gaussian distribution
$\mathcal{N}(0,1)$.  {\rr Denote the $\ell_{\infty, 2}$ norm of
$Q\in\mathbb{R}^{m\times n}$ as $\|Q\|_{\infty, 2} = \max_{j\in
\{1,\cdots, n\}}\|Q_{j}\|$ where $Q_j$ is the $j^{th}$ column of
$Q$.}

\subsection{Related Work}
\citet{CandesENR10} proposed the following formulation to solve the
problem in this paper:
\begin{equation}\label{eqn_sc}
\begin{aligned}
\min_{\theta}:&~\|D\theta\|_1~~~~s.t.:~\|\Phi\theta - c\|\leq
\varepsilon,
\end{aligned}
\end{equation}
where $D\in\mathbb{R}^{m\times p}$ is {\rcb assumed to have
orthogonal columns and $\varepsilon$ is taken as the upper bound of
$\|\epsilon\|$. They showed that the estimate error is bounded by
$C_0\varepsilon + C_1\| (D{\theta}^*)_{T^c}\|_1/ \sqrt{|T|}$ with
high probability if $\sqrt{n}\Phi\in\mathbb{R}^{n\times p}$ is a
Gaussian random matrix\footnote{Note that the ``Gaussian random
matrix'' defined in \citet{CandesENR10} is slightly different from
ours. In \citet{CandesENR10}, $\Phi\in\mathbb{R}^{n\times p}$ is a
Gaussian random matrix if each entry of $\Phi$ is generated from
$\mathcal{N}(0, 1/{n})$. Please refer to Section 1.5 in
\citet{CandesENR10}. Here we only restate the result in
\citet{CandesENR10} by using our definition for Gaussian random
matrices.} with $n\geq \Omega(s\log m)$, where $C_0$ and $C_1$ are
two constants. Letting $T=T_0$ and $\varepsilon=\|\epsilon\|$, the
error bound turns out to be $C_0\|\epsilon\|$. This result shows
that in the noiseless case, with high probability, the true signal
can be exactly recovered. In the noisy case, assume that
$\epsilon_i$'s ($i=1,\cdots,n$) are i.i.d. centered sub-Gaussian
random variables, which implies that $\|\epsilon\|^2$ is bounded by
$\Omega(n)$ with high probability. Note that since the measurement
matrix $\Phi$ is scaled by $1/\sqrt{n}$ from the definition of
``Gaussian random matrix'' in \citet{CandesENR10}, the noise vector
should be corrected similarly. In other words, $\|\epsilon\|^2$
should be bounded by $\Omega(1)$ rather than $\Omega(n)$, which
implies that the estimate error in \citet{CandesENR10} converges to
a constant asymptotically.}

{\rcc
\cite{Needell2012, Needell2012m} studied the formulation in Eq.~\eqref{eqn_sc} and considered the special case that $D$ is the total variance matrix corresponding to the $K$ ($\geq 2$) dimensional signal $\theta\in\mathbb{R}^{p^K}$ with the sparsity level $S=\|D\theta^*\|_0$. Their analysis shows that if the measurement matrix $\Phi$ is properly designed and the measurement number $n$ satisfies $n\geq\Omega(SK\log(p^K))$, then in the noiseless case (that is, $\epsilon=0$), the true signal can be exactly recovered with high probability; in the noisy case, the estimate error is bounded by $\|\hat{\theta}-\theta^*\|\leq \sqrt{K}\log (p^K)\|\epsilon\|$. Following the analysis above, one can see that the estimate error diverges when $p$ goes to infinity.
}


{\rcb \citet{NamaDEG12} considered the noiseless case and analyzed
the formulation
\begin{equation}\label{eqn_sc_noiseless}
\begin{aligned}
\min_{\theta}:&~\|D\theta\|_1~~~~s.t.:~\Phi\theta = c,
\end{aligned}
\end{equation}
assuming all rows of $D$ to be in the general position, that is, any
$p$ rows of $D$ are linearly independent, which is violated by the
fused LASSO. An sufficient condition was proposed to recover the
true signal $\theta^*$ using the cosparse analysis.

\citet{VaiterPGC12} also considered the formulation in
Eq.~\eqref{eqn_org_ls} but mainly gave robustness analysis for this
model using the cosparse technique. A sufficient condition
(different from \citet{NamaDEG12}) to exactly recover the true
signal was given in the noiseless case. In the noisy case, they took
$\lambda$ to be a value proportional to $\|\epsilon\|$ and proved
that the estimate error is bounded by $\Omega(\|\epsilon\|)$ under
certain conditions. However, they did not consider the Gaussian ensembles for
$\Phi$; see \citet[Section 3.B]{VaiterPGC12}. 

The fused LASSO, a special case of Eq.~\eqref{eqn_org_ls}, was also
studied recently. The sufficient condition of detecting jumping
points is given by~\citet{KolarSX09}. A special fused LASSO
formulation was considered by~\citet{Rinaldo09} in which $\Phi$ was
set to be the identity matrix and $D$ to be the combination of the
identity matrix and the total variance matrix. \citet{SharpnackRS12}
proposed and studied the edge LASSO by letting $\Phi$ be the
identity matrix and $D$ be the matrix corresponding to the edges of
a graph.

\subsection{Organization}
The remaining of this paper is organized as follows. To build up a
unified analysis framework, we simplify the
formulation~\eqref{eqn_org_ls} in Section~\ref{sec:simplification}.
The main results are presented in Section~\ref{sec:mainresults}.
Section~\ref{sec:conditionnumber} discusses the value of an important
parameter in our main results in three cases: the fused LASSO, the
random graph, and the total variance dictionary matrix. The numerical simulation is presented to verify the
relationship between the estimate error and the condition number in
Section~\ref{sec:simulation}. We conclude this paper in
Section~\ref{sec:conclusion}. All proofs are provided in Appendix.

\section{Simplification} \label{sec:simplification}

As highlighted by \citet{VaiterPGC12}, the analysis for a wide $D\in
\mathbb{R}^{m\times p}$ (that is, $p > m$) significantly differs
from a tall $D$ (that is, $p<m$). To build up a unified analysis
framework, we use the singular value decomposition (SVD) of $D$ to
simplify Eq.~\eqref{eqn_org_ls}, which leads to an equivalent
formulation.

Consider the compact SVD of $D$: $D=U\Sigma V_\beta^T$ where $U\in
\mathbb{R}^{m\times r}$, $\Sigma\in\mathbb{R}^{r\times r}$($r$ is
the rank of $D$), and $V_\beta\in \mathbb{R}^{p\times r}$. We then
construct $V_\alpha\in \mathbb{R}^{p\times (p-r)}$ such that
\[V:=\left[
  \begin{array}{cc}
    V_\alpha & V_\beta
  \end{array}
\right]\in\mathbb{R}^{p\times p}\] is a unitary matrix. Let $\beta =
V_\beta^T\theta$ and $\alpha = V_\alpha^T\theta$. These two linear
transformations split the original signal into two parts as follows:
\begin{align}
\min_{\alpha,\beta}:~&{1\over 2}\left\|\Phi\left[V_\alpha~
V_\beta\right]\left[\begin{array}{c}
V_\alpha^T\\V_\beta^T\end{array}\right]\theta-c\right\|^2 +
\lambda\|U\Sigma V_\beta^T\theta\|_1\\
\equiv& {1\over 2}\left\|\left[\Phi V_\alpha~\Phi
V_\beta\right]\left[\begin{array}{c} \alpha\\
\beta\end{array}\right]-c\right\|^2 + \lambda\|U\Sigma \beta\|_1\\
\equiv& {1\over 2}\left\|A\alpha+B\beta-c\right\|^2 + \lambda\|Z
\beta\|_1 \label{eqn_original_half}
\end{align}
where $A=\Phi V_{\alpha}\in\mathbb{R}^{n\times (p-r)}$, $B=\Phi
V_\beta\in\mathbb{R}^{n\times r}$, and
$Z=U\Sigma\in\mathbb{R}^{m\times r}$. Let $\hat\alpha, \hat\beta$ be
the solution of Eq.~\eqref{eqn_original_half}. One can see the
relationship between $\hat\alpha$ and $\hat\beta$: $\hat\alpha =
-(A^TA)^{-1}A^T(B\hat\beta-c),$\footnote{Here we assume that $A^TA$
is invertible.} which can be used to further simplify
Eq.~\eqref{eqn_original_half}:
\[
\min_{\beta}:~f(\beta):={1\over 2}\|(I-A(A^TA)^{-1}A^T)(B\beta
-c)\|^2 + \lambda\|Z\beta\|_1.
\]
Let $$X = (I-A(A^TA)^{-1}A^T)B$$and $$y = (I-A(A^TA)^{-1}A^T)c.$$We
obtain the following simplified formulation:
\begin{equation}
\min_{\beta}:~f(\beta)={1\over 2}\|X\beta -y\|^2 +
\lambda\|Z\beta\|_1, \label{eqn_submain}
\end{equation}
where $X\in\mathbb{R}^{n\times r}$ and $Z\in \mathbb{R}^{m\times
r}$.

Denote the solution of Eq.~\eqref{eqn_org_ls} as $\hat{\theta}$ and
the ground truth as $\theta^*$. One can verify $\hat\theta =
V[\hat\alpha^T~\hat\beta^T]^T$. Define
$\alpha^*:=V_\alpha^T\theta^*$ and $\beta^*:=V_\beta^T\theta^*$.
Note that unlike $\hat\alpha$ and $\hat\beta$ the following usually
does not hold: $\alpha^* = -(A^TA)^{-1}A^T(B\beta^*-c)$. Let
$h=\hat\beta - \beta^*$ and $d=\hat\alpha-\alpha^*$. We will study
the upper bound of $\|\hat\theta-\theta^*\|$ in terms of $\|h\|$ and
$\|d\|$ based on the relationship $\|\hat\theta-\theta^*\| \leq
\|h\| + \|d\|$.


\section{Main Results}\label{sec:mainresults}
This section presents the main results in this paper. The estimate
error by Eq.~\eqref{eqn_org_ls}, or equivalently
Eq.~\eqref{eqn_submain}, is given in Theorem~\ref{thm_main}:
\begin{theorem} \label{thm_main}
Define
\begin{align*}
  W_{Xh,1}:=&\rho^-_{X,Z^+}(s+l),
  \\W_{Xh,2}:=&6\sigma_{\min}^{-1}(Z)\rho^+_{X,Z^+}(s+l)\sqrt{s/l},  \\
  W_{d,1}:=&{1\over
    2}\sigma^{-1}_{\min}(A^TA)(\bar{\rho}^+(p-r, s+l+p-r)-\bar{\rho}^-(p-r, s+l+p-r)), \\
  W_{d,2}:=&{3\over
    2}\sigma_{\min}^{-1}(A^TA)\sigma_{\min}^{-1}(Z)\sqrt{s/l}(\bar{\rho}^+(p-r,l+p-r)-\bar{\rho}^-(p-r, l+p-r)), \\
  W_{\sigma}:=&\frac{\sigma_{\max}(Z)\sigma_{\min}(Z)}{\sigma_{\min}(Z)-3\sqrt{s/l}\sigma_{\max}(Z)}, \\
  W_h:=&3\sqrt{s/l}\sigma^{-1}_{\min}(Z),
\end{align*}
where $\bar\rho^+(p-r,.)$ and $\bar\rho^-(p-r,.)$ denote
$\rho^+_{[A,B], Z^+}(p-r,.)$ and $\rho^-_{[A,B], Z^+}(p-r,.)$
respectively for short.  Taking $\lambda >
2\|(Z^+)^TX^T\epsilon\|_\infty$ in Eq.~\eqref{eqn_org_ls}, we have
if $A^TA$ is invertible (apparently, $n\geq p-r$ is required) and
there exists an integer $l > 9 \kappa^2 s$ such that $W_{Xh, 1} -
W_{Xh,2}W_\sigma>0$, then
\begin{equation}
\begin{aligned}
\|\hat\theta -\theta^*\|\leq W_\theta \sqrt{s} \lambda +
\|(A^TA)^{-1}A^T\epsilon\| \label{eqn_thm_main}
\end{aligned}
\end{equation}
where
\[
W_\theta =
6\frac{(1+W_{d,1})W_{\sigma}+(W_h+W_{d,2})W_{\sigma}^2}{W_{Xh,
1}-W_{Xh,2}W_{\sigma}}.\]
\end{theorem}
One can see from the proof that the first term
of~\eqref{eqn_thm_main} is mainly due to the estimate error of the
sparse part $\beta$ and the second term is due to the estimate error
of the free part $\alpha$.

The upper bound in Eq.~\eqref{eqn_thm_main} strongly depends on
parameters about $X$ and $Z^+$ such as $\rho^+_{X,Z^+}(\cdot)$,
$\rho^-_{X,Z^+}(\cdot)$, $\bar{\rho}^+(\cdot,\cdot)$, and
$\bar\rho^-(\cdot,\cdot)$. Although for a given $\Phi$ and $D$, $X$
and $Z^+$ are fixed, it is still challenging to evaluate these
parameters. Similar to existing literature like \citet{CandesT05},
we assume $\Phi$ to be a Gaussian random matrix and estimate the
values of these parameters in Theorem~\ref{thm_rhobound}.

\begin{theorem} \label{thm_rhobound} Assume that $\Phi$ is a Gaussian
  random matrix. The following holds with probability at least
  $1-2\exp\{-\Omega(k\log(em/k))\}$:
\begin{equation}
   \sqrt{\rho^+_{X,Z^+}(k)}\leq \sqrt{n+r-p}+\Omega\left(\sqrt{k\log(em/k)}\right)\label{eqn_rho+}
   \end{equation}
   \begin{equation}
   \sqrt{\rho^-_{X,Z^+}(k)}\geq \sqrt{n+r-p} -
   \Omega\left(\sqrt{k\log(em/k)}\right)\label{eqn_rho-}
\end{equation}
\begin{equation}
\begin{aligned}
   &\sqrt{\rho^+_{[A,B],Z^+}(p-r, k)}\leq \sqrt{n}+\Omega(\sqrt{k+p-r}) + \Omega(\sqrt{k\log(em/k)}) \label{eqn_barrho+}
\end{aligned}
\end{equation}
\begin{equation}
\begin{aligned}
  &\sqrt{\rho^-_{[A,B],Z^+}(p-r, k)}\geq \sqrt{n} - \Omega(\sqrt{k+p-r}) - \Omega(\sqrt{k\log(em/k)}). \label{eqn_barrho-}
\end{aligned}
\end{equation}
\end{theorem}

Now we are ready to analyze the estimate error given in
Eq.~\eqref{eqn_thm_main}. Two cases are considered in the following:
the noiseless case $\epsilon=0$ and the noisy case $\epsilon\neq 0$.

\subsection{Noiseless Case $\epsilon=0$}
First let us consider the noiseless case. {\rcb Since $\epsilon =
0$, the second term in Eq.~\eqref{eqn_thm_main} vanishes. We can
choose a value of $\lambda$ to make the first term in
Eq.~\eqref{eqn_thm_main} arbitrarily small. Hence the true signal
$\theta^*$ can be recovered with an arbitrary precision as long as
$W_\theta>0$, which is equivalent to requiring
$W_{Xh,1}-W_{Xh,2}W_{\sigma} > 0$. Actually, when $\lambda$ is
extremely small, Eq.~\eqref{eqn_org_ls} approximately solves the
problem in Eq.~\eqref{eqn_sc_noiseless} with $\varepsilon=0$.}

Intuitively, the larger the measurement number $n$ is, the easier
the true signal $\theta^*$ can be recovered, since more measurements
give a feasible subspace with a lower dimension. In order to
estimate how many measurements are required, we consider the
measurement matrix $\Phi$ to be a Gaussian random matrix (This is
also a standard setup in compressive sensing.). Since this paper
mainly focuses on the large scale case, one can treat the value of
$l$ as a number proportional to $\kappa^2s$.

{\rcb Using Eq.~\eqref{eqn_rho+} and Eq.~\eqref{eqn_rho-}, we can
estimate the lower bound of $W_{Xh, 1} - W_{Xh,2} W_{\sigma}$ in
Lemma~\ref{lem_noiseless0}.
\begin{lemma} \label{lem_noiseless0}
Assume $\Phi$ to be a Gaussian random matrix. Let $l=\lceil(10\kappa)^2s\rceil$.
With probability at least $1-2\exp\{-\Omega((s+l)\log(em/(s+l)))\}$,
we have
\begin{equation}
 \begin{aligned}
&W_{Xh, 1} - W_{Xh,2} W_{\sigma} \geq {1\over 7}(n+r-p) - \Omega
  \left(\sqrt{(n+r-p)(s+l)\log\left({em\over s+l}\right)}\right).
\end{aligned} \label{eqn_lem_gauassian}
\end{equation}
\end{lemma}
From Lemma~\ref{lem_noiseless0}, to recover the true signal, we only
need
\begin{align}
 (n+r-p) > \Omega ((s+l)\log(em/(s+l))).
\label{eqn_noiseless_measure}
\end{align}
To simplify the
discussion, we propose several minor conditions first in
Assumption~\ref{ass_1}.
\begin{assumption} \label{ass_1}
Assume that
 \begin{itemize}
  \item $p-r\leq \phi n~(\phi<1)$ in the noiseless case and $p-r\leq \Omega(s)$ in the noisy case \footnote{This assumption indicates that the free dimension of the true signal $\theta^*$ (or the dimension of the free part $\alpha\in\mathbb{R}^{p-r}$) should not be too large. Intuitively, one needs more measurements to recover the free part because it has no sparse constraint and much fewer measurements to recover the sparse part. Thus, if only limited measurements are available, we have to restrict the dimension of the free part.};
  \item the condition number $\kappa={\sigma_{\max}(D)\over
      \sigma_{\min}(D)} = {\sigma_{\max}(Z)\over \sigma_{\min}(Z)} $
    is bounded;
  \item $m=\Omega(p^i)$ where $i>0$, that is, $m$ can be a polynomial function in terms of $p$.
\end{itemize}
\end{assumption}
One can verify that under Assumption~\ref{ass_1}, taking
$l=\lceil(10\kappa)^2s\rceil=\Omega(s)$, the right hand side
of~\eqref{eqn_lem_gauassian} is greater than
$$\Omega(n)-\Omega(\sqrt{ns\log (em/s)})=\Omega(n)-\Omega(\sqrt{ns\log (ep/s)}).$$Letting
$n\geq\Omega(s\log (ep/s))$ [or $\Omega(s\log (em/s))$ if without
assuming $m=\Omega(p^i)$], one can have that
$$W_{Xh, 1} - W_{Xh,2} W_{\sigma}\geq \Omega(n)-\Omega(\sqrt{ns\log (ep/s)}) > 0$$holds with high
probability (since the probability in Lemma~\ref{lem_noiseless0}
converges to 1 while $p$ goes to infinity).
In other words, in the noiseless case the true signal can be
recovered at an arbitrary precision with high probability.

To compare with existing results, we consider two special cases:
$D=\bold{I}_{p\times p}$ \citep{CandesT05} and $D$ has orthogonal
columns \citep{CandesENR10}, that is, $D^TD=I$. When
$D=\bold{I}_{p\times p}$ and $\Phi$ is a Gaussian random matrix, the
required measurements in \citet{CandesT05} are $\Omega(s\log(ep/s))$, which is the same as ours. Also note that if $D=\bold{I}_{p\times
p}$, Assumption~\ref{ass_1} is satisfied automatically. Thus our
result does not enforce any additional condition and is consistent
with existing analysis for the special case $D=\bold{I}_{p\times
p}$. Next we consider the case when $D$ has orthogonal columns as in
\citet{CandesENR10}. In this situation, all conditions except
$m=\Omega(p^i)$ in Assumption~\ref{ass_1} are satisfied. One can
easily verify that the required measurements to recover the true
signal are $\Omega(s\log (em/s))$ without assuming $m=\Omega(p^i)$
from our analysis above, which is consistent with the result in
\citet{CandesENR10}.
}

{\rcc
In addition, from Eq.~\eqref{eqn_noiseless_measure}, one can see that the boundedness requirement for $\kappa$ can be removed as long as we choose the measurements number as $n=\Omega(\kappa^2s\log (ep/s))$.
}

{\rcb
\subsection{Noisy Case $\epsilon\neq 0$} \label{sec:mainresult_noisy}

Next we consider the noisy case, that is, {\rcb study the upper
bound in~\eqref{eqn_thm_main} while $\epsilon\neq 0$. Similarly, we
mainly focus on the large scale case and assume Gaussian ensembles
for the measurement matrix $\Phi$. Theorem~\ref{thm_gaussian}
provides the upper bound of the estimate error under the conditions
in Assumption~\ref{ass_1}.}

\begin{theorem} \label{thm_gaussian}
Assume that the measurement matrix $\Phi$ is a Gaussian random
matrix, the measurement satisfies $n=O(s\log p)$, and
Assumption~\ref{ass_1} holds. Taking
$\lambda=C\|(Z^+)^TX^T\epsilon\|_\infty$ with $C>2$ in
Eq.~\eqref{eqn_org_ls}, we have
\begin{equation}
\|\hat\theta-\theta^*\|\leq \Omega\left(\sqrt{s\log p\over
        {n}}\right), \label{eqn_thm_gaussian}
\end{equation}
with probability at least
$1-\Omega(p^{-1})-\Omega(m^{-1})-\exp\{\Omega(-s\log (ep/s))\}$.
\end{theorem}
One can verify that when $p$ goes to infinity, the upper bound
in~\eqref{eqn_thm_gaussian} converges to $0$ from $n=O(s\log p)$ and
the probability converges to $1$ due to $m=\Omega(p^i)$. It means
that the estimate error converges to $0$ asymptotically given the
measures $n=O(s\log p)$.


This result shows the consistency property, that is, if the
measurement number $n$ grows faster than $s\log (p)$, the estimate
error will vanish. This consistency property is consistent with the
special case LASSO by taking $D=\bold{I}_{p\times p}$
\citep{zhang09a}. \citet{CandesENR10} considered Eq.~\eqref{eqn_sc}
and obtained an upper bound for the estimate error
$\Omega(\varepsilon/\sqrt{n})$ which does not guarantee the
consistency property like ours since
$\varepsilon=\Omega(\|\epsilon\|)=\Omega(\sqrt{n})$. Their result
only guarantees that the estimation error bound converges to a
constant given $n=O(s\log p)$.

{\rca In addition, from the derivation of
Eq.~\eqref{eqn_thm_gaussian}, one can simply verify that the
boundedness requirement for $\kappa$ can actually be removed, if we
allow more observations, for example, $ n=O(\kappa^4 s\log p)$. Here
we enforce the boundedness condition just for simplification of
analysis and a convenient comparison to the standard LASSO (it needs
$n=O(s\log p)$ measurements). } }

\section{The Condition Number of $D$} \label{sec:conditionnumber}
Since $\kappa$ is a key factor from the derivation of
Eq.~\eqref{eqn_thm_gaussian}, we consider the fused LASSO and the
random graphs and estimate the values of $\kappa$ in these two
cases.

\begin{figure*}[t!]
  \centering
  \includegraphics[width=0.3\textwidth]{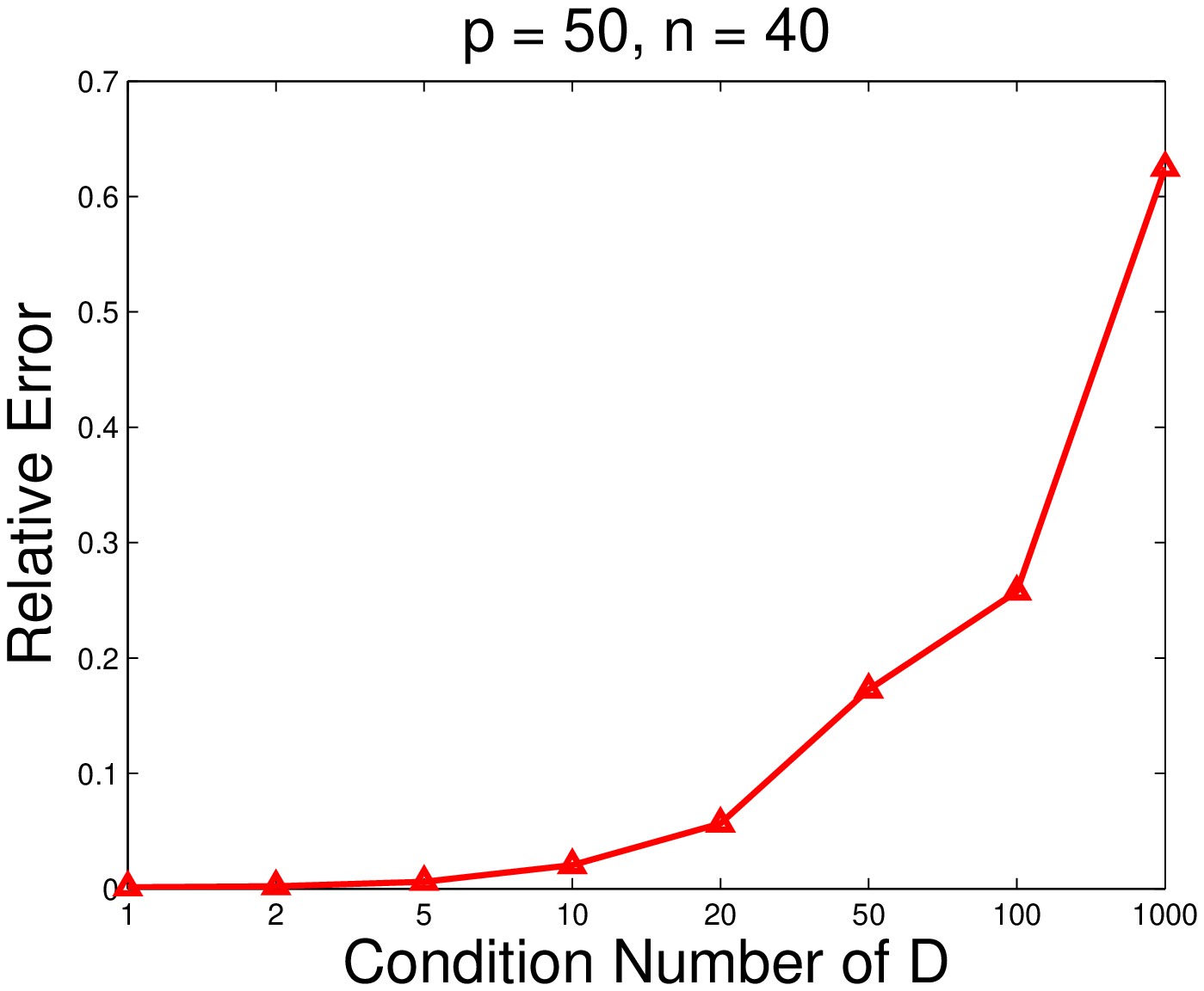}
  \includegraphics[width=0.3\textwidth]{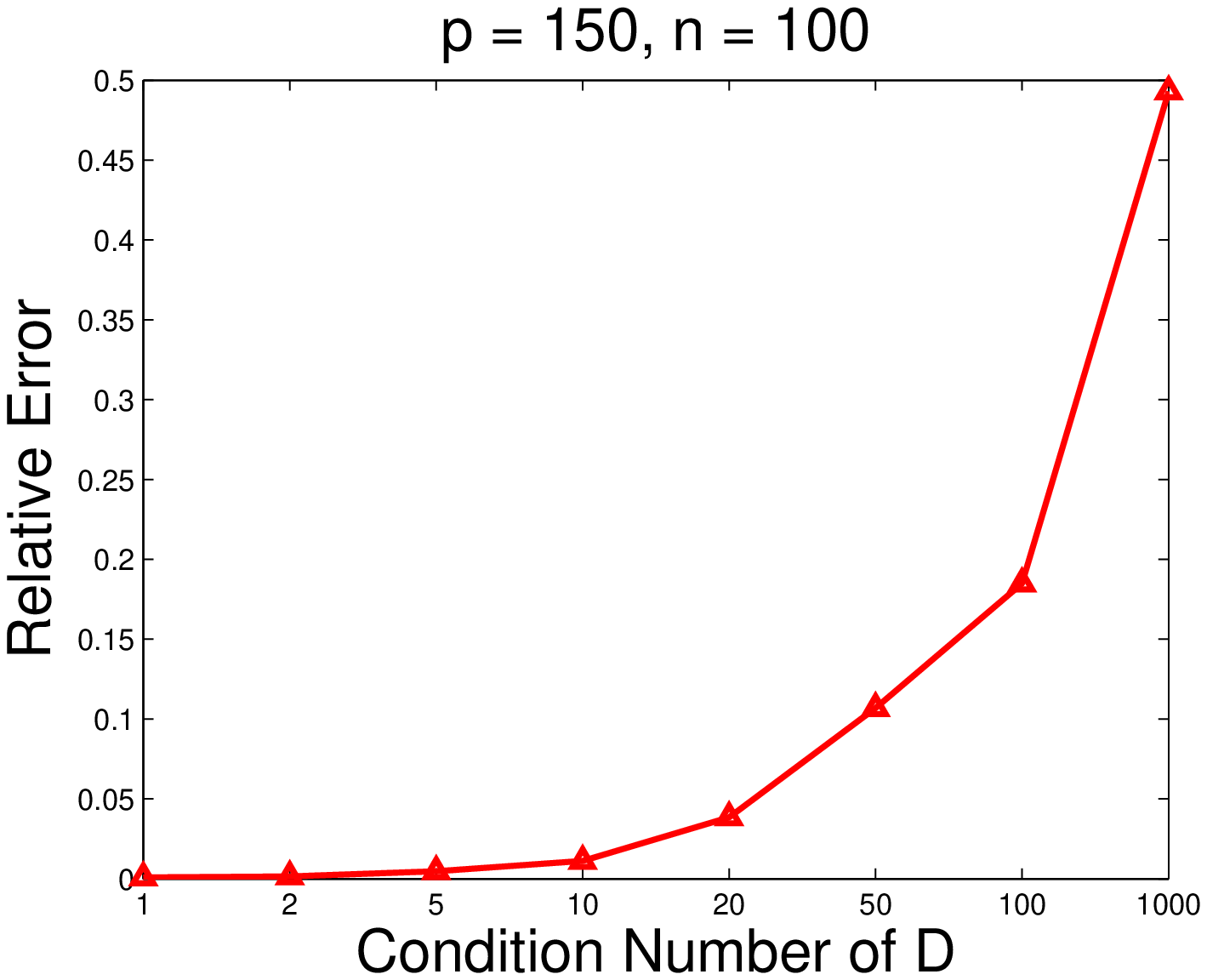}
  \includegraphics[width=0.3\textwidth]{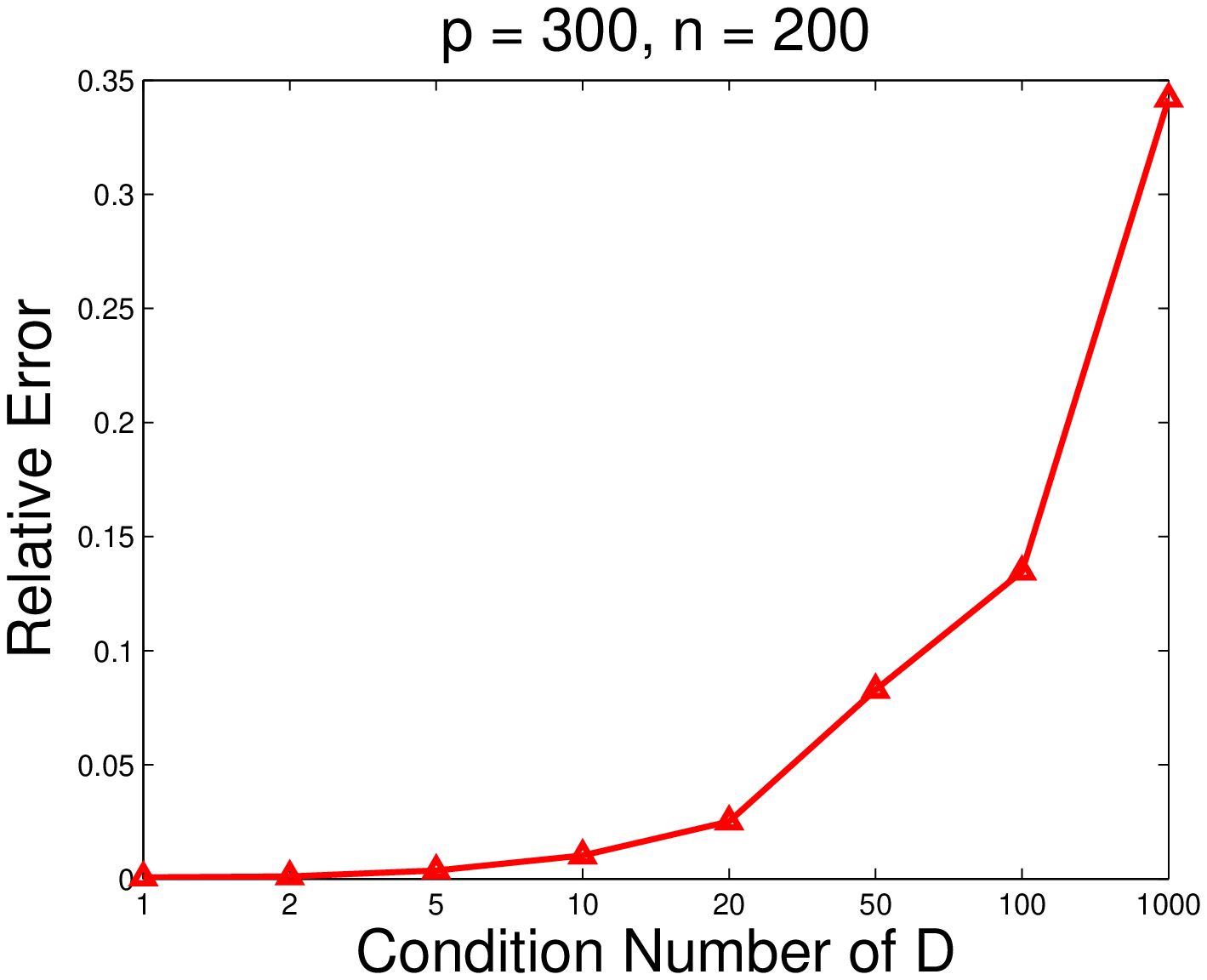}
  \caption{Illustration of the relationship between condition number
    and performance in terms of relative error. Three problem sizes
    are used as examples.}
  \label{fig:perf-cond}
\end{figure*}

Let us consider the fused LASSO first. The transformation matrix $D$
is
$$\left[\begin{array}{c}
\left[\begin{array}{cc} \bold{I}_{(p-1)\times (p-1)}&\bold{0}_{p-1}
\end{array}\right]-\left[\begin{array}{cc} \bold{0}_{p-1}& \bold{I}_{(p-1)\times (p-1)}
\end{array}\right]\\
\bold{I}_{p\times p}\end{array} \right].$$ One can verify that
$$\sigma_{\min}(D) = \min_{\|v\|=1}\|Dv\|\geq  \min_{\|v\|=1} \|v\| =
1$$and
\begin{equation*}
\begin{aligned}
&\sigma_{\max}(D) = \max_{\|v\|=1}\|Dv\|\\
\leq& \max_{\|v\|=1}\|\left[\begin{array}{cc} \bold{I}_{(p-1)\times
(p-1)}&\bold{0}_{p-1}
\end{array}\right]v-\left[\begin{array}{cc} \bold{0}_{p-1}& \bold{I}_{(p-1)\times
(p-1)}
\end{array}\right]v\| + \|v\|\\
\leq & \max_{\|v\|=1}\|\left[\begin{array}{cc} \bold{I}_{(p-1)\times
(p-1)}&\bold{0}_{p-1}
\end{array}\right]\|\|v\|+\|\left[\begin{array}{cc} \bold{0}_{p-1}& \bold{I}_{(p-1)\times
(p-1)}
\end{array}\right]\|\|v\| + \|v\|\\
\leq &3
\end{aligned}
\end{equation*}
which implies that $\sigma_{\min}(D)\geq 1$ and $\sigma_{\max}(D)
\leq 3$. Hence we have $\kappa\leq 3$ in the fused LASSO case.

Next we consider therandom graph. The transformation matrix $D$
corresponding to a random graph is generated in the following way:
(1) each row is independent of the others;
(2) two entries of each row are uniformly selected and are set to
$1$ and $-1$ respectively;
(3) the remaining entries are set to $0$.
The following result shows that the condition number of $D$ is
bounded with high probability.
\begin{theorem}\label{thm_kappa}
  For any $m$ and $p$ satisfying that $m\geq cp$ where $c$ is large
  enough, the  following holds:
$${\sigma_{\max}(D)\over \sigma_{\min}(D)} \leq {\sqrt{m}+\Omega(\sqrt{p})\over \sqrt{m}-\Omega(\sqrt{p})},$$with probability at least $1-2\exp\{-\Omega(p)\}$.
\end{theorem}
From this theorem, one can see that
\begin{itemize}
\item If $m=cp$ where $c$ is large enough, then $$\kappa =
  {\sigma_{\max}(Z)\over \sigma_{\min}(Z)} = {\sigma_{\max}(D)\over
    \sigma_{\min}(D)}$$is bounded with high probability;
\item If $m=p(p-1)/2$ which is the maximal possible $m$, then $\kappa \rightarrow 1$.
\end{itemize}

{\rcc
We consider the last special case for $D$ as the total variance matrix corresponding to the $K$ dimensional signal $\theta\in \mathbb{R}^{p^K}$. In general, the condition number of $D$ is unbounded. Comparing with results in \citet{Needell2012, Needell2012m} which focus on this particular case and need a special measurement matrix, our results still have advantages in some aspects, for example, when the measurements satisfy $n=O(\kappa^4 s \log p)$, the estimate error from our analysis converges to zero while it diverges given the same number of measurements from their results.
}

\section{Numerical Simulations} \label{sec:simulation}

In this section, we use numerical simulations to verify some of our
theoretical results. Given a problem size $n$ and $p$ and condition
number $\kappa$, we randomly generate $D$ as follows. We first
construct a $p \times p$ diagonal matrix $D_0$ such that
$$\text{Diag}(D_0) > 0~~\text{and}~~\frac{\max(\text{Diag}(D_0))}{\min(\text{Diag}(D_0))} = \kappa.$$We
then construct a random basis matrix $V \in \mathbb{R}^{p \times
p}$, and let $D = D_0 V$. Clearly, $D$ has independent columns and
the condition number equals to $\kappa$. Next, a vector $x \in
\mathbb{R}^p$ is generated such that $x_i \sim \mathcal{N}(0, 1)$,
$i = 1, \ldots, \frac{p}{10}$ and $x_j = 0$, $j = \frac{p}{10} + 1,
\ldots, p$. $\theta^*$ is then obtained as $\theta^* = D^{-1}x$.
Finally, we generate a matrix $\Phi \in \mathbb{R}^{n \times p}$
with $\Phi_{ij} \sim \mathcal{N}(0, 1)$, noise $\epsilon \in
\mathbb{R}^n$ with $\epsilon_i \sim \mathcal{N}(0, 0.001)$ and $y =
\Phi \theta^* + \epsilon$.

We solve Eq.~\eqref{eqn_org_ls} using the standard optimization
package CVX\footnote{\url{cvxr.com/cvx/}} and $\lambda$ is set as
$\lambda = 2 \| (Z^+)^T X^T \epsilon \|_{\infty}$ as suggested by
Theorem~\ref{thm_main}. We use three different sizes of problems,
with $n \in \{40, 100, 200\}$, $p \in \{50, 150, 300\}$ and $\kappa$
ranging from 1 to 1000. For each problem setting, 100 random
instances are generated and the average performance is reported. We
use the relative error $\frac{\| \hat{\theta} - \theta^*
\|}{\|\theta^*\|}$ for evaluation, and present the performance with
respect to different condition numbers in
Figure~\ref{fig:perf-cond}. We can observe from
Figure~\ref{fig:perf-cond} that in all three cases the relative
error increases when the condition number increases. If we fix the
condition number, by comparing the three curves, we can see that the
relative error decreases when the problem size increases. These are
consistent with our theoretical results in Section 3 [see
Eq.~\eqref{eqn_thm_gaussian}].


\section{Conclusion and Future Work}\label{sec:conclusion}
{\rcb This paper considers the problem of estimating a specific type
of signals which is sparse under a given linear transformation $D$.}
A conventional convex relaxation technique is used to convert this
NP-hard combinatorial optimization into a tractable problem: {\rcc dictionary LASSO}. We
develop a unified framework to analyze the {\rcc dictionary LASSO} with a
generic $D$ and provide the estimate error bound. Our main results
establish that 1) in the noiseless case, if the condition number of
$D$ is bounded and the measurement number $n\geq \Omega(s\log(p))$
where $s$ is the sparsity number, then the true solution can be
recovered with high probability; and 2) in the noisy case, if the
condition number of $D$ is bounded and the measurement number grows
faster than $s\log (p)$ [that is, $s\log(p)=o(n)$], then the
estimate error converges to zero when $p$ and $s$ go to infinity
with probability 1. Our results are consistent with existing
literature for the special case $D=\bold{I}_{p\times p}$
(equivalently LASSO) and improve the existing analysis for the same
formulation. The condition number of $D$ plays a critical role in
our theoretical analysis. We consider the condition numbers in two
cases including the fused LASSO and the random graph. The condition
number in the fused LASSO case is bounded by a constant, while the
condition number in the random graph case is bounded with high
probability if $m\over p$ (that is,
$\#\text{edge}\over\#\text{vertex}$) is larger than a certain
constant. Numerical simulations are consistent with our theoretical
results.

{\rca In future work, we plan to study a more general formulation of
Eq.~\eqref{eqn_org_ls}:
$$\min_{\theta}:~f(\theta) + \lambda\|D\theta\|_1,$$where $D$ is an
arbitrary matrix and $f(\theta)$ is a convex and smooth function
{\rcb satisfying the restricted strong convexity property}. We
expect to obtain similar consistency properties for this general
formulation.}

\section*{Acknowledgments}
{
This work was supported in part by NSF grants IIS-0953662 and MCB-1026710. We would like to sincerely thank Professor Sijian Wang and
Professor Eric Bach of the University of Wisconsin-Madison for useful discussion and helpful advice.
}


\bibliographystyle{plainnat}
\bibliography{referenceCS787}


\appendix
\section*{Appendix A. Proof of Theorem~\ref{thm_main}}
We first introduce several important
definitions used in the proof. We divide the complementary index set
$T_0^c:=\{1,2,...,m\}\backslash T_0$ into a group of subsets $T_j$'s
($j=1,2,\cdots,J$), without intersection, such that $T_1$ indicates
the index set of the largest $l$ entries of $Z_{T_0^c}h$ (in the
absolute value), $T_2$ contains the next-largest $l$ entries of
$Z_{T_0^c}h$, and so forth\footnote{The last subset may contain
fewer  than $l$ elements.}. $T_0\cup T_1$ is denoted as $T_{01}$ for short.

First we give the proof skeleton of Theorem~\ref{thm_main}. Recall
that the estimate error $\|\hat\theta-\theta^*\|$ is bounded by the
sum of the free part error $\|d\|$ (that is,
$\|\hat\alpha-\alpha^*\|$) and the sparse part error $\|h\|$ (that
is, $\|\hat\beta-\beta^*\|$). Lemma~\ref{lem_d} and
Lemma~\ref{lem_h} studied the upper bound of $\|h\|$ and $\|d\|$
respectively and the proof of Theorem~\ref{thm_main} makes use of
these two upper bounds.
%

\begin{assumption}\label{ass_feasible}
Assume that
\begin{equation}
\|(Z^+)^TX^T(X\beta^*-y)\|_\infty=\|(Z^+)^TX^T\epsilon\|_\infty <
\lambda/2. \label{eqn_feasible}
\end{equation}
\end{assumption}

\begin{lemma}\label{lem_T0T0c}
Assume that Assumption~\ref{ass_feasible} holds. We have
\[
3\|Z_{T_0}h\|_1 \geq \|Z_{T_0^c}h\|_1.
\]
\end{lemma}

\begin{proof}
Since $\hat\beta$ is the optimal solution of
Eq.~\eqref{eqn_submain}, we have
\begin{align*}
0&\geq{1\over 2}\|X \hat{\beta} - {y}\|^2 - {1\over 2}\|X
\beta^*
-{y}\|^2 + \lambda(\|Z\hat{\beta}\|_1 - \|Z\beta^*\|_1)\\
&=\left[X \left({\hat{\beta}+\beta^* \over 2}\right) -
{y}\right]^T X h +
\lambda(\|Z\hat{\beta}\|_1 - \|Z\beta^*\|_1)\\
&=\left[X \left({\beta^*+ h/2}\right) - {y}\right]^T X h +
\lambda(\|Z\hat{\beta}\|_1 - \|Z\beta^*\|_1)\\
&\geq \left[X \beta^*- {y}\right]^T X h +
\lambda(\|Z\hat{\beta}\|_1 - \|Z\beta^*\|_1)\\
&= h^TZ^T(Z^+)^TX^T(X \beta^* -{y}) +
\lambda(\|Z_{T_0}\hat{\beta}\|_1-\|Z_{T_0}\beta^*\|_1 +
\|Z_{T_0^c}\hat{\beta}\|_1 - \|Z_{T_0^c}\beta^*\|_1)\\
&\geq -\|Zh\|_1\|(Z^+)^TX^T(X \beta^* -{y})\|_\infty +
\lambda(\|Z_{T_0}\hat{\beta}\|_1-\|Z_{T_0}\beta^*\|_1 +
\|Z_{T_0^c}\hat{\beta}\|_1 - \|Z_{T_0^c}\beta^*\|_1)\\
&\geq -\|Zh\|_1\lambda/2 +
\lambda(\|Z_{T_0}\hat{\beta}\|_1-\|Z_{T_0}\beta^*\|_1 +
\|Z_{T_0^c}\hat{\beta}\|_1)~~~~(\text{from~Assumption~\ref{ass_feasible}})\\
&\geq -(\|Z_{T_0}h\|_1+\|Z_{T_0^c}h\|_1)\lambda/2 + \lambda(-\|Z_{T_0}h\|_1+\|Z_{T_0^c}\hat\beta\|_1)\\
&\geq {1\over 2}\lambda\|Z_{T_0^c}h\|_1 - {3\over
2}\lambda\|Z_{T_0}h\|_1.
\end{align*}
It completes the proof.
\end{proof}

\begin{lemma}
For any matrices $P$, $Q$, $Z$, and $X$ with compatible dimensions
and $k,l_3>0$, we have
\begin{subequations}
\begin{align}
\max_{v\in \mathcal{H}(Z,l_3)} {\|P^TQv\|\over \|v\|} \leq& {1\over
2}\left(\rho^+_{[P,Q], Z}(k, l_3)-\rho^-_{[P,Q], {Z}}(k,l_3)\right)
\end{align}
\end{subequations}
where $k$ denotes the number of columns of $P$.
\end{lemma}
\begin{proof}
The claim follows from
\begin{align*}
&\max_{v\in \mathcal{H}(Z,l_3)}{\|P^TQv\|\over \|v\|} \\
=&\max_{u\in \mathbb{R}^{k}, v\in \mathcal{H}(Z,l_3)} {|u^TP^TQv|\over \|u\|\|v\|} \\
=&\max_{\|u\|=1, \|v\|=1, u\in \mathbb{R}^k, v\in \mathcal{H}(Z,l_3)} {|u^TP^TQv|} &\\
=&\max_{\|u\|=1, \|v\|=1, u\in \mathbb{R}^k, v\in
\mathcal{H}(Z,l_3)} {1\over 4}\left| \left\|[P,Q]\left[
                                                                \begin{array}{c}
                                                                  u \\
                                                                  v \\
                                                                \end{array}
                                                              \right]\right\|^2
                                                              -\left\|[P,Q]\left[
                                                                \begin{array}{c}
                                                                  u \\
                                                                  -v \\
                                                                \end{array}
                                                              \right]\right\|^2
\right| \\
\leq &\max_{\|u\|=1, \|v\|=1, u\in \mathbb{R}^k, v\in
\mathcal{H}(Z,l_3)} {1\over 4}\left(\rho^+_{[P,Q],
Z}(k,l_3)\left\|\left[
                                                                \begin{array}{c}
                                                                  u \\
                                                                  v \\
                                                                \end{array}
                                                              \right]\right\|^2 - \rho^-_{[P,Q], {Z}}(k, l_3)\left\|\left[
                                                                \begin{array}{c}
                                                                  u \\
                                                                  -v \\
                                                                \end{array}
                                                              \right]\right\|^2\right)\\
\leq &{1\over 2}\left(\rho^+_{[P,Q], {Z}}(k, l_3)-\rho^-_{[P,Q],
{Z}}(k,l_3)\right).
\end{align*}
\end{proof}

\begin{lemma}\label{lem_sumTj_T0}
Assume that Assumption~\ref{ass_feasible} holds. We have
\begin{equation}
\sum_{j\geq 2}\|Z_{T_j}h\| \leq 3\sqrt{s/l}\|Z_{T_0}h\|.
\end{equation}
\begin{proof}
From the LHS, we have
\begin{align*}
\sum_{j\geq 2}\|Z_{T_j}h\| = &\sum_{j\geq 2} \sqrt{\|Z_{T_j}h\|^2}\\
\leq & \sum_{j\geq 2} \sqrt{l(\|Z_{T_{j-1}}h\|_1/l)^2}\\
\leq & \|Z_{T_0^c}h\|_1/\sqrt{l}\\
\leq &3\|Z_{T_0}h\|_1/\sqrt{l}~~~~(\text{from~Lemma~\ref{lem_T0T0c}})\\
\leq & 3\sqrt{s/l}\|Z_{T_0}h\|.
\end{align*}
It completes the proof.
\end{proof}
\end{lemma}


\begin{lemma}
Assume that Assumption~\ref{ass_feasible} holds. We have
\begin{equation}
\|Xh\|^2 \geq W_{Xh,1}\|Z^+_{T_{01}}Z_{T_{01}}h\|^2 -
W_{Xh,2}\|Z^+_{T_{01}}Z_{T_{01}}h\|\|Z_{T_{01}}h\|,
\end{equation}
where $W_{Xh,1}$ and $W_{Xh,2}$ are defined in
Theorem~\ref{thm_main}.
\end{lemma}
\begin{proof}
The inequality is derived from
\begin{align*}
&\|Xh\|^2=|h^TZ^T(Z^+)^TX^TXZ^+Zh| \\
\geq &h^TZ^T_{T_{01}}(Z^+_{T_{01}})^TX^TXZ^+_{T_{01}}Z_{T_{01}}h -
2\sum_{j\geq
2}|h^TZ^T_{T_{01}}(Z^+_{T_{01}})^TX^TXZ^+_{T_{j}}Z_{T_{j}}h|\\
\geq &h^TZ^T_{T_{01}}(Z^+_{T_{01}})^TX^TXZ^+_{T_{01}}Z_{T_{01}}h -
2\|XZ^+_{T_{01}}Z_{T_{01}}h\|\sum_{j\geq
2}\|XZ^+_{T_{j}}Z_{T_{j}}h\|\\
\geq &\rho^-_{X,Z^+}(s+l)\|Z^+_{T_{01}}Z_{T_{01}}h\|^2 -2
\sqrt{\rho^+_{X,Z^+}(s+l)}\sqrt{\rho^+_{X,Z^+}(l)}\|Z^+_{T_{01}}Z_{T_{01}}h\|\sum_{j\geq
2}\|Z^+_{T_{j}}Z_{T_{j}}h\|\\
\geq &\rho^-_{X,Z^+}(s+l)\|Z^+_{T_{01}}Z_{T_{01}}h\|^2 -
2\rho^+_{X,Z^+}(s+l)\|Z^+_{T_{01}}Z_{T_{01}}h\|\sigma_{\min}^{-1}(Z)\sum_{j\geq
2}\|Z_{T_{j}}h\|\\ \geq
&\rho^-_{X,Z^+}(s+l)\|Z^+_{T_{01}}Z_{T_{01}}h\|^2 -6
\rho^+_{X,Z^+}(s+l)\|Z^+_{T_{01}}Z_{T_{01}}h\|\sigma_{\min}^{-1}(Z)\|Z_{T_0}h\|/\sqrt{s/l}\quad
(\text{from Lemma~\ref{lem_sumTj_T0}})\\ \geq
&\rho^-_{X,Z^+}(s+l)\|Z^+_{T_{01}}Z_{T_{01}}h\|^2 -
6\sigma_{\min}^{-1}(Z)\rho^+_{X,Z^+}(s+l)\|Z^+_{T_{01}}Z_{T_{01}}h\|\|Z_{T_{01}}h\|\sqrt{s/l}.
\end{align*}
It completes the proof.
\end{proof}

\begin{lemma}
Assume that Assumption~\ref{ass_feasible} holds. We have
\begin{equation}
\|Xh\|^2 \leq 6\sqrt{s}\lambda \|Z_{T_{01}}h\|.
\end{equation}
\end{lemma}

\begin{proof}
From the optimality condition, we have that there exists $g$
satisfying $Z^Tg\in
\partial \|Z\hat{\beta}\|_1$ and $\|g\|_{\infty} \le 1$ such that
\begin{align*}
&X^T(X\hat\beta-y) = -\lambda Z^Tg\\
\Rightarrow &X^T(I-A(A^TA)^{-1}A^T)(B\hat\beta-c)= -\lambda Z^Tg~~~~(\text{due to the definition of $X$})\\
\Rightarrow &X^T(I-A(A^TA)^{-1}A^T)(B\hat\beta-A\alpha^*-B\beta^*-\epsilon)= -\lambda Z^Tg\\
\Rightarrow &X^T(I-A(A^TA)^{-1}A^T)(B\hat\beta-B\beta^*-\epsilon)= -\lambda Z^Tg\\
\Rightarrow &X^T(X(\hat\beta-\beta^*)-\epsilon) = -\lambda Z^Tg\\
\Rightarrow &X^TXh = -\lambda Z^Tg+X^T\epsilon\\
\Rightarrow &h^TX^TXh =-\lambda h^TZ^Tg +h^TX^T\epsilon\\
\Rightarrow &\|Xh\|^2 \leq \lambda \|Zh\|_1 \|g\|_\infty + \|Zh\|_1
\|(Z^+)^TX^T\epsilon\|_\infty \\
\Rightarrow &\|Xh\|^2 \leq {3\over 2}\lambda \|Zh\|_1 \leq {3\over
2}\lambda (\|Z_{T_0}h\|_1 + \|Z_{T_0^c}h\|_1) \leq 6\lambda
\|Z_{T_0}h\|_1 \leq 6\sqrt{s}\lambda \|Z_{T_{01}}h\|.
\end{align*}
It completes the proof.
\end{proof}

\begin{lemma} \label{lem_d}
Assume that Assumption~\ref{ass_feasible} holds. We have
\begin{align*}
\|d\|=&\|\hat\alpha-\alpha^*\|\leq
W_{d,1}\|Z^+_{T_{01}}Z_{T_{01}}h\| + W_{d,2}\|Z_{T_0}h\| +
\|(A^TA)^{-1}A^T\epsilon\|,
\end{align*}
where $W_{d,1}$ and $W_{d,2}$ are defined in Theorem~\ref{thm_main}.
\end{lemma}

\begin{proof}
Noticing that $\hat{\alpha} = -(A^TA)^{-1}A^T(B\hat\beta-c)$, we
have
\begin{align*}
& \hat{\alpha}= -(A^TA)^{-1}A^T(B\hat\beta-A\alpha^* - B\beta^* - \epsilon)\\
  \Rightarrow & \hat\alpha - \alpha^*= -(A^TA)^{-1}A^T(B(\hat\beta-\beta^*) -
  \epsilon)\\
  \Rightarrow &d= -(A^TA)^{-1}A^T(Bh -
  \epsilon)
\end{align*}
It follows that
\begin{equation} \label{eqn_lem6.1}
\|d\| \leq \sigma_{\min}^{-1}(A^TA)\|A^TBh\|
+\|(A^TA)^{-1}A^T\epsilon\|.
\end{equation}

Consider $\|A^TBh\|$ as follows:
\begin{align*}
\|A^TBh\|=& \|A^T(BZ^+_{T_{01}}Z_{T_{01}}+\sum_{j\geq
2}BZ^+_{T_j}Z_{T_j})h\| \\
\leq & \|A^TBZ^+_{T_{01}}Z_{T_{01}}h\|+\sum_{j\geq
2}\|A^TBZ^+_{T_j}Z_{T_j}h\| \\
\leq & {1\over 2}(\bar{\rho}^+(p-r, s+l)-\bar{\rho}^-(p-r,
s+l))\|Z^+_{T_{01}}Z_{T_{01}}h\|
+ \\
&{1\over
2}(\bar{\rho}^+(p-r, l)-\bar{\rho}^-(p-r, l))\sum_{j\geq2}\|Z^+_{T_{j}}Z_{T_{j}}h\|\\
\leq & {1\over 2}(\bar{\rho}^+(p-r, s+l)-\bar{\rho}^-(p-r,
s+l))\|Z^+_{T_{01}}Z_{T_{01}}h\| +\\ &{\sigma_{\min}^{-1}(Z)\over
2}(\bar{\rho}^+(p-r, l)-\bar{\rho}^-(p-r,l))\sum_{j\geq2}\|Z_{T_{j}}h\|\\
\leq & {1\over 2}(\bar{\rho}^+(p-r, s+l)-\bar{\rho}^-(p-r,
s+l))\|Z^+_{T_{01}}Z_{T_{01}}h\| +\\ &{3\sigma_{\min}^{-1}(Z)\over
2}(\bar{\rho}^+(p-r,l)-\bar{\rho}^-(p-r,l)) \|Z_{T_0}h\|\sqrt{s/l}.
\end{align*}
The last inequality is due to Lemma~\ref{lem_sumTj_T0}. Plugging it
into Eq.~\eqref{eqn_lem6.1}, we obtain the claim.
\end{proof}

\begin{lemma} \label{lem_h}
Assume that Assumption~\ref{ass_feasible} holds. For any integer $l >
(3\kappa)^2s$, we have that
\begin{subequations}
\begin{align}
\|h\|\leq & \|Z^+_{T_{01}}Z_{T_{01}}h\| + W_{h}\|Z_{T_{01}}h\|;\\
\|Z_{T_{01}}h\|\leq & W_{\sigma}\|Z^+_{T_{01}}Z_{T_{01}}h\|; \\
\|Z^+_{T_{01}}Z_{T_{01}}h\| \leq & \frac{W_{\sigma}6\sqrt{s}\lambda}{W_{Xh, 1}-W_{Xh,2}W_{\sigma}};\\
\|Z_{T_{01}}h\| \leq & \frac{W_{\sigma}^26\sqrt{s}\lambda}{W_{Xh,
1}-W_{Xh,2}W_{\sigma}},
\end{align}
\end{subequations}
where $W_{\sigma}$ and $W_h$ are defined in Theorem~\ref{thm_main}.
\end{lemma}
\begin{proof}
The first inequality is obtained from
\begin{align*}
\|h\| = \|Z^+Zh\| \leq \|Z^+_{T_{01}}Z_{T_{01}}h\| + \sum_{j\geq
2}\|Z^+_{T_j}Z_{T_j}h\| \leq \|Z^+_{T_{01}}Z_{T_{01}}h\| +
3\sqrt{s/l}\sigma_{\min}^{-1}(Z) \|Z_{T_{01}}h\|.
\end{align*}
It follows that
\begin{align*}
&\sigma_{\max}^{-1}(Z)\|Z_{T_{01}}h\| \leq \sigma_{\min}(Z^+)\|Zh\| \leq \|Z^+Zh\| \leq \|h\| \leq \|Z^+_{T_{01}}Z_{T_{01}}h\|  + {3\sqrt{s/l}\sigma_{\min}^{-1}(Z)} \|Z_{T_{01}}h\| \\
\Rightarrow &\|Z_{T_{01}}h\| \leq (\sigma_{\max}^{-1}(Z)-3\sqrt{s/l}\sigma_{\min}^{-1}(Z))^{-1}\|Z^+_{T_{01}}Z_{T_{01}}h\|\\
\Rightarrow &\|Z_{T_{01}}h\| \leq
\frac{\sigma_{\max}(Z)\sigma_{\min}(Z)}{\sigma_{\min}(Z)-3\sqrt{s/l}\sigma_{\max}(Z)}\|Z^+_{T_{01}}Z_{T_{01}}h\|
= W_{\sigma}\|Z^+_{T_{01}}Z_{T_{01}}h\|
\end{align*}
which implies the second inequality. The third inequality is
satisfied automatically if $\|Z^+_{T_{01}}Z_{T_{01}}h\|=0$. We only
need to prove the situation $\|Z^+_{T_{01}}Z_{T_{01}}h\|\neq 0$:
\begin{align*}
&W_{Xh, 1}\|Z^+_{T_{01}}Z_{T_{01}}h\|^2  -  W_{Xh,2}\|Z_{T_{01}}h\|\|Z^+_{T_{01}}Z_{T_{01}}h\| \leq 6\sqrt{s}\lambda \|Z_{T_{01}}h\| \leq 6\sqrt{s}W_{\sigma}\lambda \|Z^+_{T_{01}}Z_{T_{01}}h\|\\
\Rightarrow &W_{Xh, 1}\|Z^+_{T_{01}}Z_{T_{01}}h\|  -  W_{Xh,2}\|Z_{T_{01}}h\| \leq 6\sqrt{s}W_{\sigma}\lambda\\
\Rightarrow &W_{Xh, 1}\|Z^+_{T_{01}}Z_{T_{01}}h\|  -  W_{\sigma}W_{Xh,2}\|Z^+_{T_{01}}Z_{T_{01}}h\| \leq 6\sqrt{s}W_{\sigma}\lambda\\
\Rightarrow & \|Z^+_{T_{01}}Z_{T_{01}}h\| \leq
\frac{W_{\sigma}}{W_{Xh, 1}-W_{Xh,2}W_{\sigma}}6\sqrt{s}\lambda.
\end{align*}
The last claim is from the combination of the second and third
inequalities.
\end{proof}


\noindent{\bf Proof of Theorem~\ref{thm_main}}
\begin{proof}
Applying Lemma~\ref{lem_d} and Lemma~\ref{lem_h}, we obtain
\begin{align*}
\|\hat\theta-\theta^*\|\leq&\|d\| + \|h\|\\
\leq& (1+W_{d,1})\|Z^+_{T_{01}}Z_{T_{01}}h\| + (W_h+W_{d,2})\|Z_{T_{01}}h\| + \|(A^TA)^{-1}A^T\epsilon\|\\
\leq & \frac{(1+W_{d,1})W_{\sigma}+(W_h+W_{d,2})W_{\sigma}^2}{W_{Xh,
1}-W_{Xh,2}W_{\sigma}}6\sqrt{s}\lambda +\|(A^TA)^{-1}A^T\epsilon\| \\
=& W_\theta\sqrt{s}\lambda + \|(A^TA)^{-1}A^T\epsilon\|.
\end{align*}
It completes the proof.
\end{proof}

\section*{Appendix B. Proofs of Theorem~\ref{thm_rhobound}, Theorem~\ref{thm_gaussian}, and Theorem~\ref{thm_kappa}}

\begin{lemma}\label{lem_Qepsilon2}
For any $Q\in\mathbb{R}^{n\times (p-r)}$, we have
\begin{align}
\mathbb{P}\left(\|Q^T\epsilon\|\leq
\Omega(\|Q\|_F\Delta\sqrt{\log{p}})\right) > 1- \Omega\left({1\over
p}\right).
\end{align}
\end{lemma}
\begin{proof}
Since $\epsilon_i$'s are i.i.d. centered sub-Gaussian noise with
sub-Gaussian norm $\Delta$, from \citet[Proposition 10.2]{zhang09a},
one has that
\begin{align*}
  \mathbb{P}\left(\|Q^T\epsilon\|> \|Q\|_F(\Omega(\Delta)+t)\right) \leq \exp\left\{-{t^2\over
  \Omega(\Delta^2)}\right\}.
\end{align*}
Taking $t=\Omega(\Delta\sqrt{\log{p}})$, we obtain that
\[
\mathbb{P}\left(\|Q^T\epsilon\|>
\Omega(\|Q\|_F\Delta\sqrt{\log{p}})\right) \leq \Omega\left({1\over
p}\right),
\]
which indicates the claim.
\end{proof}

\begin{lemma} \label{lem_Qepsilon}
For any matrix $Q\in \mathbb{R}^{n\times m}$, we have
\begin{align*}
\mathbb{P}(\|Q^T\epsilon\|_\infty\leq
\Omega(\|Q\|_{\infty,2}\Delta\sqrt{\log{(em)}})) >1-\Omega\left(1\over
m\right).
\end{align*}
\end{lemma}

\begin{proof}
Since $\epsilon_i$'s are i.i.d. centered sub-Gaussian noise with
sub-Gaussian norm $\Delta$, $Q_j^T\epsilon$ is centered sub-Gaussian
random variable with sub-Gaussian norm $\Omega(\|Q_j\|\Delta)$ where
$Q_j$ is the $j^{th}$ column of $Q$. Using Hoeffding-type
inequality~\citep[see][Lemma 5.9]{Vershynin11} and the property of
sub-Gaussian random variables, we obtain
\begin{equation*}
\mathbb{P}(|Q^T_j\epsilon| > t) \leq
\exp\left\{1-\Omega\left({t^2\over
\|Q_j\|^2\Delta^2}\right)\right\},
\end{equation*}
which indicates that
\begin{align*}
\mathbb{P}(\|Q^T\epsilon\|_\infty>t)=&\mathbb{P}(\max_{j}|Q^T_j\epsilon|>t)\\
 \leq & \sum_{j=1}^m\exp\left\{1-\Omega\left({t^2\over
\|Q_j\|^2\Delta^2}\right)\right\} \\
\leq &m\exp\left\{1-\Omega\left({t^2\over
\max_j\|Q_j\|^2\Delta^2}\right)\right\}.
\end{align*}
Taking $t=\Omega(\max_j\|Q_j\|\Delta\sqrt{\log{(em)}})${\rr(the
factor in front of $\max_j\|Q_j\|\Delta\sqrt{\log{(em)}}$ should be
large enough, particularly, at least $\sqrt{2}$ times the factor in
front of ${t^2\over \max_j\|Q_j\|^2\Delta^2}$)}, we have
$$\mathbb{P}(\|Q^T\epsilon\|_\infty>\Omega(\max_j\|Q_j\|\Delta\sqrt{\log{(em)}})) \leq
\Omega\left({1\over m}\right),$$which implies the claim.
\end{proof}

\begin{lemma}\label{lem_lambdabound}
Assume that $\Phi$ is a Gaussian random matrix. With probability at
least $1-\Omega\left({1\over m}\right)$, we have
\begin{equation*}
\|(Z^+)^TX^T(X\beta^*-y)\|_\infty=\|(Z^+)^TX^T\epsilon\|_\infty\leq
\lambda/2 
\end{equation*}
where
$\lambda=\Omega\left(\Delta\sigma^{-1}_{\min}(Z)\left(\sqrt{n+r-p}+((n+r-p)\log(m))^{1/4}\right)\sqrt{\log(em)}\right).$
\end{lemma}

\begin{proof}
First from 
\begin{align*}
X^T(X\beta^*-y) = &X^T(I-A(A^TA)^{-1}A^T)(B\beta^*-c)\\
=& X^T(I-A(A^TA)^{-1}A^T)(-A\alpha^*-
\epsilon)\\
=& -X^T\epsilon,
\end{align*} 
we prove the first part of the claim
$\|(Z^+)^TX^T(X\beta^*-y)\|_\infty=\|(Z^+)^TX^T\epsilon\|_\infty$.

Let $Z^+_j$ be the $j^{th}$ column of $Z^+$, $I-A(A^TA)^{-1}A =
PP^T$ where $P\in\mathbb{R}^{n\times (n+r-p)}$ has orthogonal
columns, and $Y=P^T\Phi V_\beta\in\mathbb{R}^{(n+r-p)\times r}$. One
can verify that $Y$ is a Gaussian random matrix. Using Eq.~(3.2)
in~\citet{MendelsonPT08}, we have
\begin{align*}
&\mathbb{P}\left[{\|XZ^+_j\|^2\over n+r-p} - \|Z^+_j\|^2 \geq t\|Z^+_j\|^2\right] \\
=&\mathbb{P}\left[{\|(I-A(A^TA)^{-1}A)\Phi V_\beta Z^+_j\|^2\over n+r-p} - \|Z^+_j\|^2 \geq t\|Z^+_j\|^2\right] \\
=&\mathbb{P}\left[{\|PP^T\Phi V_\beta Z^+_j\|^2\over n+r-p} - \|Z^+_j\|^2 \geq t\|Z^+_j\|^2\right]\\
=&\mathbb{P}\left[{\|Y Z^+_j\|^2\over n+r-p} - \|Z^+_j\|^2 \geq
t\|Z^+_j\|^2\right]\leq \exp\{-\Omega((n+r-p)t^2)\}
\end{align*}
It follows that
\begin{align*}
&\mathbb{P}\left[{\|XZ^+_j\|^2\over \|Z^+_j\|^2} \geq (1+t)(n+r-p)\right] \leq \exp\{-\Omega((n+r-p)t^2)\}\\
\Rightarrow&\mathbb{P}\left[\max_{j}{\|XZ^+_j\|^2\over \|Z^+_j\|^2} \geq (1+t)(n+r-p)\right] \leq m\exp\{-\Omega((n+r-p)t^2)\}\\\Rightarrow &\mathbb{P}\left[\max_{j}{\|XZ^+_j\|} \geq \sqrt{(1+t)(n+r-p)}\max_j\|Z^+_j\|\right] \leq m\exp\{-\Omega((n+r-p)t^2)\}\\
\Rightarrow &\mathbb{P}\left[{\|XZ^+\|_{\infty,2}} \geq \sqrt{(1+t)(n+r-p)}\sigma^{-1}_{\min}(Z)\right] \leq m\exp\{-\Omega((n+r-p)t^2)\}\\
\Rightarrow &\mathbb{P}\left[{\|XZ^+\|_{\infty,2}} \geq \sqrt{(1+t)(n+r-p)}\sigma^{-1}_{\min}(Z)\right] \leq \exp\{\log(m)-\Omega((n+r-p)t^2)\}\\
\end{align*}
Taking $t=\Omega(\sqrt{\log(p)/(n+r-p)})$, we obtain $$
\mathbb{P}\left[{\|XZ^+\|_{\infty,2}} \geq
\left(\sqrt{n+r-p}+\Omega\left((n+r-p)\log(m)\right)^{1/4}\right)\sigma^{-1}_{\min}(Z)\right]
\leq \Omega\left({1\over m}\right).$$ Applying
Lemma~\ref{lem_Qepsilon}, we obtain
\begin{align*}
&\mathbb{P}\left[\|(Z^+)^TX^T\epsilon\|_\infty \geq \Omega\left(\|XZ^+\|_{\infty,2}\Delta\sqrt{\log(em)}\right)\right] \leq \Omega\left(1\over m\right)\\
\Rightarrow&\mathbb{P}\left[\|(Z^+)^TX^T\epsilon\|_\infty \geq \left(\sqrt{n+r-p}+\Omega\left((n+r-p)\log(m)\right)^{1/4}\right)\sigma^{-1}_{\min}(Z)\Omega\left(\Delta\sqrt{\log(em)}\right)\right] \\
&\leq \Omega\left(1\over m\right) \\
\Rightarrow&\mathbb{P}\left[\|(Z^+)^TX^T\epsilon\|_\infty \geq \Omega\left(\Delta\sigma^{-1}_{\min}(Z)\left(\sqrt{n+r-p}+\left((n+r-p)\log(m)\right)^{1/4}\sqrt{\log(em)}\right)\right)\right] \\
&\leq \Omega\left(1\over m\right) 
\end{align*}
which implies the claim.
\end{proof}

\noindent {\bf Proof of Theorem~\ref{thm_rhobound}}
\begin{proof}
First one can verify that for any matrices $P$ and $Q$ with
orthogonal columns, $P^T\Phi Q$ is a Gaussian random matrix. Hence
$A$ and $B$ are Gaussian matrices. Let $(I-A(A^TA)^{-1}A^T) = PP^T$
where $P\in\mathbb{R}^{n\times(n+r-p)}$ has orthogonal columns. Let
$F\subset\{1, \cdots, m\}$ be an index set with cardinality $|F|=k$.
Let $Q\in \mathbb{R}^{r\times k}$ have orthogonal columns, whose
image is the subspace spanned by columns of $Z^+$ in the index set
$F$.

\begin{align*}
& \mathbb{P}\left[\max_{h\in \mathcal{H}(Z^+_F,k)}{\|Xh\|\over \|h\|} > \sqrt{n+r-p} + \Omega(\sqrt{k})+t\right]\\
=& \mathbb{P}\left[\max_{h\in \mathcal{H}(Z^+_F,k)}{\|(I-A(A^TA)^{-1}A)Bh\|\over \|h\|}> \sqrt{n+r-p} + \Omega(\sqrt{k})+t\right]\\
=& \mathbb{P}\left[\max_{v}{\|PP^T\Phi V_\beta Qv\|\over \|Qv\|}> \sqrt{n+r-p} + \Omega(\sqrt{k})+t\right]\\
=& \mathbb{P}\left[\max_{v}{\|P^T\Phi V_\beta Qv\|\over \|Qv\|}> \sqrt{n+r-p} + \Omega(\sqrt{k})+t\right]\\
=& \mathbb{P}\left[\max_{v}{\|Yv\|\over \|v\|}> \sqrt{n+r-p} +
\Omega(\sqrt{k})+t\right] \\
\leq & 2\exp(-\Omega(t^2)).
\end{align*}
where $Y=P^T\Phi V_{\beta}\in\mathbb{R}^{(n+r-p)\times r}$ is Gaussian random matrix and the last inequality uses Theorem 5.39  \citep{Vershynin11}.
Since $\mathcal{H}(Z^+, k) = \cup_{F\subset
\{1,\cdots,m\}}\mathcal{H}(Z^+_F,k)$, we have
\begin{align*}
 \mathbb{P}\left[\max_{h\in \mathcal{H}(Z^+, k)}{\|Xh\|\over \|h\|} > \sqrt{n+r-p}+\Omega(\sqrt{k})+t \right] \leq& \left(\begin{array}{c}
                                                                     m\\k
                                                                    \end{array}
\right)2\exp(-\Omega(t^2)) \leq 2\exp(k\log(em/k)-\Omega(t^2)).
\end{align*}
Taking $t=\Omega(\sqrt{k\log(em/k)})$, we obtain
\begin{align*}
&\mathbb{P}\left[\max_{h\in \mathcal{H}(Z^+, k)}{\|Xh\|\over \|h\|} > \sqrt{n+r-p}+\Omega(\sqrt{k\log(em/k)})\right]\\
=& \mathbb{P}\left[\sqrt{\rho^+_{X,Z^+}(k)}>
\sqrt{n+r-p}+\Omega(\sqrt{k\log(em/k)})\right]\leq
2\exp(-\Omega(k\log(em/k))).
\end{align*}
Similarly, we have $\mathbb{P}\left[\sqrt{\rho^-_{X,Z^+}(k)}>
\sqrt{n+r-p}-\Omega(\sqrt{k\log(em/k)})\right] \leq
2\exp(-\Omega(k\log(em/k))).$ Denote
\begin{align*}
        \bar{Y}=\Phi V\left[\begin{array}{cc}
                        \bold{I}_{(p-r)\times(p-r)}&0\\
              0&Q
                       \end{array}\right] \in \mathbb{R}^{n\times (p-r+k)},
\end{align*}
which is a Gaussian random matrix. We have
\begin{align*}
 & \mathbb{P}\left[\max_{h\in \mathbb{R}^{p-r}\times\mathcal{H}(Z^+_F,k)}{\|\left[A~B\right] h\|\over \|h\|} > \sqrt{n} + \Omega(\sqrt{k+p-r})+t\right]\\
= & \mathbb{P}\left[\max_{h\in \mathbb{R}^{p-r}\times\mathcal{H}(Z^+_F,k)}{\|\Phi Vh\|\over \|h\|} > \sqrt{n} + \Omega(\sqrt{k+p-r})+t\right]\\
= & \mathbb{P}\left[\max_{u\in \mathbb{R}^{p-r}, v\in \mathbb{R}^{k}}{\left\|\Phi V\left[\begin{array}{cc}\bold{I}&0\\0&Q\end{array}\right]\left[\begin{array}{c}u\\v\end{array}\right]\right\|/ \left\|\left[\begin{array}{cc}\bold{I}&0\\0&Q\end{array}\right]\left[\begin{array}{c}u\\v\end{array}\right]\right\|} > \sqrt{n} + \Omega(\sqrt{k+p-r})+t\right]\\
= & \mathbb{P}\left[\max_{u\in \mathbb{R}^{p-r}, v\in
\mathbb{R}^{k}}{\left\|\bar{Y}\left[\begin{array}{c}u\\v\end{array}\right]\right\|/\left\|\left[\begin{array}{c}u\\v\end{array}\right]\right\|}
> \sqrt{n} + \Omega(\sqrt{k+p-r})+t\right] \leq
2\exp(-\Omega(t^2)).
\end{align*}
Since $
\mathbb{R}^{p-r}\times\mathcal{H}(Z^+,k)=\cup_{F\subset\{1,\cdots,p\}}
\mathbb{R}^{p-r}\times\mathcal{H}(Z^+_F,k)$, we have
\begin{align*}
&\mathbb{P}\left[\rho^+_{\left[A~B\right], Z^+}(p-r, k) > \sqrt{n} + \Omega(\sqrt{k+p-r})+t\right]\\
=&\mathbb{P}\left[\max_{h\in \mathbb{R}^{p-r}\times\mathcal{H}(Z^+,k)}{\|\left[A~B\right] h\|\over \|h\|} > \sqrt{n} + \Omega(\sqrt{k+p-r})+t\right]\\
\leq & \left(\begin{array}{c}
              m\\
k
             \end{array}
\right)2\exp(-\Omega(t^2)) \leq 2\exp\{k\log(em/k)
-\Omega(t^2)\},
\end{align*}
Taking $t=\Omega(\sqrt{k\log(em/k)})$, we obtain the
third claim. The proof of the last inequality can be obtained
similarly.
\end{proof}

\noindent {\bf Proof of Lemma~\ref{lem_noiseless0}}
\begin{proof}
Using Eq.~\eqref{eqn_rho+} and Eq.~\eqref{eqn_rho-}, we have
\begin{align*}
  &W_{Xh, 1} - W_{Xh,2} W_{\sigma} \\
 =& \rho^-_{X,Z^+}(s+l) - 6\sigma^{-1}_{\min}(Z)\rho^+_{X,Z^+}(s+l)\sqrt{s/l}\frac{\sigma_{\max}(Z)\sigma_{\min}(Z)}{\sigma_{\min}(Z)-3\sqrt{s/l}\sigma_{\max}(Z)}\\\
  \geq &\rho^-_{X,Z^+}(s+l) - {6\over \sqrt{{l\over(s\kappa^2)}}-3}\rho^+_{X,Z^+}(s+l)\\
  \geq & \rho^-_{X,Z^+}(s+l) - {6\over 7}\rho^+_{X,Z^+}(s+l) \\
  \geq & (n+r-p) - \Omega \left(\sqrt{(n+r-p)(s+l)\log\left({em\over s+l}\right)}\right) \\
  &-{6\over 7}\left[(n+r-p) + \Omega \left(\sqrt{(n+r-p)(s+l)\log(em/l)}\right)\right] \\
  =& {1\over 7}(n+r-p) - \Omega
  \left(\sqrt{(n+r-p)(s+l)\log\left({em\over s+l}\right)}\right)
\end{align*}
holds with probability at least $1-2\exp\{-\Omega((s+l)\log(em/(s+l)))\}$. 
\end{proof}

\begin{lemma}\label{lem_Ainverse}
Assume that $A\in\mathbb{R}^{n\times (p-r)}$ with $n> (p-r)$ and
$\Phi$ is a Gaussian random matrix.  We have
\begin{align}
\mathbb{P}\left[\|(A^TA)^{-1}A^T\|_F \leq  {\sqrt{p-r}\over
\sqrt{n}-\sqrt{p-r} -\Omega(\sqrt{\log n})}\right] \geq &
1-\Omega\left(1\over n\right)\\
\mathbb{P}\left[\sigma_{\min}(A^TA) \leq
n-\Omega(\sqrt{n(p-r)})\right] \geq & 1-\Omega\left(1\over n\right).
\end{align}
\end{lemma}
\begin{proof}
Since $A=\Phi V_\alpha$ and $V_\alpha$ has orthogonal columns, we
know that $A\in \mathbb{R}^{n \times (p-r)}$ is a Gaussian random
matrix. Denote $\sigma_i(A)$ as the $i^{th}$ largest positive
singular value. Then we have
\begin{align*}
\|(A^TA)^{-1}A^T\|^2_F = \sum_{i=1}^{p-r}\sigma_i^{-2}(A) \leq
\sum_{i=1}^{p-r}\sigma_{\min}^{-2}(A)=(p-r)\sigma_{\min}^{-2}(A).
\end{align*}
Using Corollary~5.35~\citep{Vershynin11}, we have
\begin{equation}\label{eqn_lem:Ainverse.1}
\mathbb{P}[\sigma_{\min}(A) \geq
\sqrt{n}-\sqrt{p-r}-\Omega(\sqrt{\log n})] \geq
1-\Omega\left({1\over n}\right).
\end{equation} \
Hence the first claim follows from
\begin{align*}
&\mathbb{P}\left[\|(A^TA)^{-1}A^T\|_F \leq  {\sqrt{p-r}\over \sqrt{n}-\sqrt{p-r} -\Omega(\sqrt{\log n})}\right] \\
\geq & \mathbb{P}\left[\sqrt{p-r}\sigma_{\min}^{-1}(A)
\leq {\sqrt{p-r}\over \sqrt{n}-\sqrt{p-r} -\Omega(\sqrt{\log n})} \right]\\
= & \mathbb{P}\left[\sigma_{\min}(A) \geq \sqrt{n}-\sqrt{p-r}
-\Omega(\sqrt{\log n}) \right] \geq 1-\Omega\left(1\over n\right).
\end{align*}
The second inequality is obtained directly from
Eq.~\eqref{eqn_lem:Ainverse.1} with the following relationship:
$\sigma_{\min}(A^TA) = \sigma_{\min}^2(A)$.
\end{proof}

\noindent {\bf Proof of Theorem~\ref{thm_gaussian}}
\begin{proof}
Let $l=\lceil(10\kappa)^2s\rceil$. First let us consider the second term of
Eq.~\eqref{eqn_thm_main}. Using Lemma~\ref{lem_Qepsilon2} and
Lemma~\ref{lem_Ainverse}, we have that with probability at least
$1-\Omega(p^{-1})$, the following holds
\begin{align*}
  \left\|(A^TA)^{-1}A^T\epsilon \right\| \leq {\sqrt{p-r}\sqrt{\log
      p}\Delta\over \sqrt{n}-\sqrt{p-r}-\Omega(\sqrt{\log n})}
      \leq
  \Omega\left({\sqrt{(p-r)\log p}\over \sqrt{n}}\right)\leq
  \Omega\left({s\sqrt{\log p}\over \sqrt{n}}\right).
\end{align*}
Now we consider the first term of Eq.~\eqref{eqn_thm_main}. Using
Eq.~\eqref{eqn_barrho+} and Eq.~\eqref{eqn_barrho-}, we derive the
following:
\begin{align*}
  W_{d,1} =& \Omega \biggl( \sigma^{-1}_{\min}(A^TA)\Big(
      \rho^+_{[A,B], Z^+}(p-r,s+l+p-r)-\rho^-_{[A,B], Z^+}(p-r,s+l+p-r)\Big) \biggr) \\
  \leq& \Omega\left(\sqrt{n(s+l+2(p-r))\log{ep\over s+l+p-r}}\over
    n-\sqrt{n(p-r)}\right)\leq \Omega\left(\sqrt{s\log{ep\over s}\over
      n}\right)
\end{align*}
with probability at least $1-\exp\{\Omega(-s\log(ep/s))\}$.
Similarly,  with the same probability, we have
$W_{d,2}=\Omega\left(\sigma^{-1}_{\min}(Z)\sqrt{s\log{(ep/s)}\over
n}\right)$. $W_{\sigma}$ and $W_h$ are
bounded by $\Omega(\sigma_{\max}(Z))$ and
$\Omega(\sigma_{\min}^{-1}(Z))$ respectively. From
Lemma~\ref{lem_noiseless0}, we have $W_{Xh,
1}-W_{Xh,2}W_{\sigma}\geq \Omega(n)$. Now we are ready to estimate
the upper bound of $W_\theta$ with holding probability at least
$1-\Omega(p^{-1})-\exp\{\Omega(-s\log (ep/s))\}$:
\begin{align*}
W_\theta =
&6\frac{(1+W_{d,1})W_{\sigma}+(W_h+W_{d,2})W_{\sigma}^2}{W_{Xh,
1}-W_{Xh,2}W_{\sigma}}\\
\leq & \Omega\left({\sigma_{max}(Z)\left(1+\sqrt{s\log(ep/s)\over n}\right) + \sigma_{min}^{-1}(Z)\sigma_{max}^2(Z)\left(1+\sqrt{s\log(ep/s)\over n}\right) \over n}\right)\\
= &\Omega \left(n^{-1}\left(\sigma_{max}(Z) +
\sigma_{min}^{-1}(Z)\sigma_{max}^2(Z)\right)\left(1+\sqrt{s\log(ep/s)\over
n}\right)\right).
\end{align*}

Next let us consider the value of $\lambda$. From
Lemma~\ref{lem_lambdabound}, we have
\begin{align*}
\lambda\leq&\Omega\biggl(\Delta\sigma^{-1}_{\min}(Z)\sqrt{\log(em)}\left(\sqrt{n+r-p}+((n+r-p)\log(p))^{1/4}\right)\biggr)\\
  \leq&\Omega \left(\sigma^{-1}_{\min}(Z)\sqrt{n\log m}\right)\\
  \leq&\Omega\left(\sigma^{-1}_{\min}(Z)\sqrt{n\log p}\right)
\end{align*}
with probability at least $1-\Omega(p^{-1})-\Omega(m^{-1})$.

Finally we can express the estimate bound in
Eq.~\eqref{eqn_thm_main} as
\begin{equation}
\begin{aligned}
  &\|\theta^*-\hat{\theta}\|\\
  \leq& \Omega \left(n^{-1}\left(\sigma_{max}(Z) + \sigma_{min}^{-1}(Z)\sigma_{max}^2(Z)\right)\left(1+\sqrt{s\log(ep/s)\over n}\right)\sigma_{\min}^{-1}(Z)\sqrt{sn\log p}\right)+ \Omega\left(\sqrt{s\log p\over {n}}\right)\\
  =&\Omega\Biggl((\kappa+\kappa^2)\left(1+\sqrt{s\log(ep/s)\over
        n}\right)\sqrt{s\log p\over {n}}+\sqrt{s\log p\over
        {n}}\Biggl)\\
=&\Omega\Biggl(\sqrt{s\log p\over
        {n}}\Biggl)\label{eqn_thetabound}
\end{aligned}
\end{equation}
with probability at least
$1-\Omega(p^{-1})-\Omega(m^{-1})-\exp\{\Omega(-s\log (ep/s))\}$.
\end{proof}

\noindent {\bf Proof of Theorem~\ref{thm_kappa}}
\begin{proof}
Denote each row of $D$ as $d_k^T$, $k=1,\cdots, m$. The manner to
generate $d_k$ indicates that all $d_k$'s are independent and
$\mathbb{E}(d_{ki}^2) = {2\over p}$ and $\mathbb{E}(d_{ki}d_{kj}) =
{-2\over p(p-1)}$ for any $i\neq j$. Hence we have
\begin{equation}
Q:={p\over 2}\mathbb{E}(d_kd^T_k) = \left[\begin{array}{cccc}
1&-{1\over p-1}&\cdots&-{1\over p-1}\\
-{1\over p-1}&1&\cdots&-{1\over p-1}\\
\cdots&\cdots&\cdots&\cdots\\
-{1\over p-1}&-{1\over p-1}&\cdots&1
\end{array}\right].
\end{equation}
One can verify that all $p-1$ nonzero eigenvalues of $Q$ are
identical and positive. Thus, we can decompose $Q$ as $Q=\gamma
U_QU_Q^T$ where $\gamma>0$ and $U_Q\in \mathbb{R}^{p\times (p-1)}$
such that $U_Q$ has orthogonal columns. Let
$\tilde{d}_k={(p/2\gamma)^{1/2}}U_Q^Td_k$. It is easy to see that all
$\tilde{d}_k$'s are independent. 
\begin{equation}
\mathbb{E}(\tilde{d}_k\tilde{d}_k^T)={p\over
2\gamma}\mathbb{E}(U_Q^Td_kd_k^TU_Q) = {1\over \gamma}U_Q^TQU_Q =
\bold{I}_{(p-1)\times(p-1)}.
\end{equation}
Hence $\tilde{d}_k$'s are independent isotropic random vectors. Next
we can verify that $\tilde{d}_k$'s are sub-Gaussian random vectors
since each entry of $\tilde{d}_k$ is bounded such that for any fixed
$x\in\mathbb{R}^{p-1}$ the inner product $\langle x, \tilde{d}_k
\rangle$ is bounded. From Definition 5.22 in \citet{Vershynin11}, we
know that $\tilde{d}_k$'s are sub-Gaussian random vectors. We can construct $\tilde{D}$ by
$\tilde{D}={(p/2\gamma)^{1/2}}DU_Q$, that is, the $k^{th}$ row of
$\tilde{D}$ is $\tilde{d}_k^T$. Using
Theorem 5.39 in \citet{Vershynin11}, we obtain that with probability
at least $1-2\exp\{-\Omega(p)\}$, one has
\begin{equation}
\sqrt{m}-\Omega(\sqrt{p}) \leq \sigma_{\min}(\tilde{D}) \leq
\sigma_{\max}(\tilde{D}) \leq \sqrt{m}+\Omega(\sqrt{p})
\end{equation}
Note that all singular values of $\tilde{D}$ are proportional to all
nonzero singular values of $D$. Hence, we have
${\sigma_{\max}(\tilde{D})\over
\sigma_{\min}(\tilde{D})}={\sigma_{\max}(D)\over \sigma_{\min}(D)}$,
which completes the proof.
\end{proof}

\end{document}